\def\IJCAIPROCEEDINGSONLY#1{}

\documentclass{article}
\usepackage{ijcai23}

\usepackage{times}
\usepackage{soul}
\usepackage{url}
\usepackage[hidelinks]{hyperref}
\usepackage[utf8]{inputenc}
\usepackage[small]{caption}
\usepackage{graphicx}
\usepackage{amsmath}
\usepackage{amsthm}
\usepackage{booktabs}
\usepackage{algorithm}
\usepackage{algorithmic}
\usepackage{amssymb}
\usepackage{mathtools}

\usepackage[capitalize,noabbrev]{cleveref}

\theoremstyle{plain}
\newtheorem{theorem}{Theorem}[section]
\newtheorem{proposition}[theorem]{Proposition}
\newtheorem{lemma}[theorem]{Lemma}
\newtheorem{corollary}[theorem]{Corollary}
\theoremstyle{definition}
\newtheorem{definition}[theorem]{Definition}

\theoremstyle{remark}
\newtheorem{remark}[theorem]{Remark}
\theoremstyle{remark}
\newtheorem{example}[theorem]{Example}

\usepackage[textsize=tiny]{todonotes}

\theoremstyle{definition}

\usepackage[nice]{nicefrac}
\usepackage{xfrac}
\usepackage{color, colortbl}
 \usepackage{stmaryrd}

\usepackage{verbatim}

\usepackage{paralist}

\usepackage{stmaryrd}

\usepackage{xfrac}

\usepackage[normalem]{ulem}

\usepackage{etoolbox}

\usepackage[switch]{lineno} 

\usepackage[small]{complexity}
\newclass{\sharpP}{\text{\rm\#}P}

\newcommand\myldots{\!\makebox[1em][c]{.\hfil.\hfil.}}  

\renewcommand{\R}{\mathbb{R}}

\makeatletter
\newcommand{\DeclareMathActive}[2]{%
  \expandafter\edef\csname keep@#1@code\endcsname{\mathchar\the\mathcode`#1 }
  \begingroup\lccode`~=`#1\relax
  \lowercase{\endgroup\def~}{#2}%
  \AtBeginDocument{\mathcode`#1="8000 }%
}

\newcommand{\std}[1]{\csname keep@#1@code\endcsname}
\patchcmd{\newmcodes@}{\mathcode`\-\relax}{\std@minuscode\relax}{}{\ddt}
\AtBeginDocument{\edef\std@minuscode{\the\mathcode`-}}
\makeatother

\let\IsInPP=\relax

\DeclareMathActive{=}{\if\IsInPP1{\std{=}}\else\mathrel{\std{=}}\fi}

\newcommand{\compactEquals}[1]{\let\IsInPP=1#1\let\IsInPP=\relax}
\renewcommand{\PP}[1]{\mathbb{P}(\compactEquals{#1})}

\newcommand{\fM}{{\mathfrak M}}

\newcommand{\fx}{{\mathfrak x}}
\newcommand{\fu}{{\mathfrak u}}
\newcommand{\ff}{{\mathfrak f}}
\newcommand{\ft}{{\mathfrak t}}

\newcommand{\cG}{{\mathcal G}}
\newcommand{\cF}{{\mathcal F}}
\renewcommand{\cL}{{\mathcal L}}

\newcommand{\bI}{{\bf I}}
\newcommand{\bL}{{\bf L}}
\newcommand{\bO}{{\bf O}}
\newcommand{\bR}{{\bf R}}
\newcommand{\bU}{{\bf U}}
\newcommand{\bV}{{\bf V}}
\newcommand{\bY}{{\bf Y}}
\newcommand{\bX}{{\bf X}}

\newcommand{\be}{{\bf e}}

\newcommand{\bi}{{\bf i}}

\newcommand{\bl}{{\bf l}}

\newcommand{\bo}{{\bf o}}

\newcommand{\br}{{\bf r}}

\newcommand{\bu}{{\bf u}}
\newcommand{\bv}{{\bf v}}

\newcommand{\by}{{\bf y}}
\newcommand{\bx}{{\bf x}}

\def\dop{\textit{do}}

\def\false{\textsc{false}}
\def\true{\textsc{true}}
\def\termdef{\equiv}
\def\pfalse{\ff}
\def\ptrue{\ft}
\def\maxvaluecount{c}

\def\complexityclass#1{{\tt #1}}
\def\ccNP{\complexityclass{NP}}
\def\ccPSPACE{\complexityclass{PSPACE}}
\def\ccEXPSPACE{\complexityclass{EXPSPACE}}
\def\ccNEXPTIME{\complexityclass{NEXPTIME}}

\newcommand{\subscript}{interven}

\newcommand{\SATprobpolysum}{\mbox{\sc Sat}_{\textit{prob}}}
\newcommand{\SATcausalpolysum}{\mbox{\sc Sat}_{\textit{\subscript}}}

\def\Lprop{\cL_{\textit{prop}}}

\def\Lint{\cL_{\textit{int}}}
\def\Lfull{\cL_{\textit{post-int}}}

\def\Tpolysum{T}

\def\Lprobpolysum{\cL_{\textit{prob}}}

\def\Lcausalpolysum{\cL_{\textit{causal}}^{\textit{poly}+\Sigma}}
\def\Lcausalpolysum{\cL_{\textit{\subscript}}}

\def\Ext{\textit{Ext}}

\newcommand{\cETR}{\ensuremath{\exists\Rset}}
\newcommand{\csuccETR}{\ensuremath{\mathtt{succ}\exists\Rset}}
\newcommand{\ETR}{\ensuremath{\mathrm{ETR}}}
\newcommand{\succETR}{\ensuremath{\mathrm{succETR}}}
\newcommand{\FEASfour}{\ensuremath{4\text{-}\mathrm{FEAS}}}
\newcommand{\succFEASfour}{\ensuremath{\mathrm{succ}\text{-}4\text{-}\mathrm{FEAS}}}
\newcommand{\QUAD}{\ensuremath{\mathrm{QUAD}}}
\newcommand{\succQUAD}{\ensuremath{\mathrm{succQUAD}}}
\newcommand{\lep}{\le_p}
\newcommand{\sigmaETR}{\ensuremath{\Sigma_{\textit{vi}}\text{-}\mathrm{ETR}}}
\newcommand{\sigmapiETR}{\ensuremath{\Sigma\Pi_{\textit{vi}}\text{-}\mathrm{ETR}}}

\newcommand{\ETRc}[1]{\ensuremath{\mathrm{ETR}^{#1,+,\times}}}
\newcommand{\succETRc}[1]{\ensuremath{\mathrm{succETR}^{#1,+,\times}}}
\newcommand{\ETRcR}[2]{\ensuremath{\mathrm{ETR}^{#1,+,\times}_{[#2]}}}
\newcommand{\succETRcR}[2]{\ensuremath{\mathrm{succETR}^{#1,+,\times}_{[#2]}}}

\newcommand{\Rset}{\mathbb{R}}

\def\mmid{ | }

\usepackage{tikz}
\usetikzlibrary{graphs, arrows.meta, shapes.misc}

\definecolor{ba.yellow}{RGB}{252,190,18}
\definecolor{ba.gray}{RGB}{153,153,156}
\definecolor{ba.blue}{RGB}{6,123,164}
\definecolor{ba.red}{RGB}{213,96,98}
\definecolor{ba.orange}{RGB}{233,116,81}
\definecolor{ba.pine}{RGB}{67,154,134}
\definecolor{ba.green}{RGB}{196,247,161}
\definecolor{ba.violet}{RGB}{88, 53, 94}

\tikzset{exp/.style={color=green!50!black}}
\tikzset{res/.style={color=blue}}
\tikzset{lat/.style={color=gray}}

\tikzset{
  pico/.style = {
    every node/.style = {
      draw,
      circle,
      semithick,
      inner sep = 0pt,
      minimum width = 0.7ex,
      fill = white
    },
    semithick
  },
  edge/.style = {
    semithick
  },
  arc/.style = {
    edge,
    ->,
    >={[round,sep]Stealth}
  },
}
\tikzset{
  axis/.style = {
    semithick,
    ->,
    >={[round,sep]Stealth},
  },
  tick/.style = {
    thin,
    font=\small
  },
  timeout/.style = {
    semithick,
    densely dashed,
    color = ba.red,
    font=\small,
  },
  mean_dot/.style = {
    draw, fill,
    circle,
    inner sep = 0pt,
    minimum width = 1mm
  }
}

\title{The Hardness of Reasoning about Probabilities and Causality}

\author{
  Benito van der Zander$^1$\and
  Markus Bl\"{a}ser$^{2,\,*}$\and
  Maciej Li\'{s}kiewicz$^{1,}$\footnote{Equally last authors.}
\affiliations
$^1$Institute of Theoretical Computer Science, University of L\"{u}beck, Germany\\
$^2$Saarland University, Saarland Informatics Campus, Saarbr\"{u}cken, Germany
\emails
   benito@tcs.uni-luebeck.de,
   mblaeser@cs.uni-saarland.de,
   liskiewi@tcs.uni-luebeck.de
}

\newcommand{\citet}{\cite}
\newcommand{\citep}{\cite}

\begin{document}

\maketitle

\begin{abstract}
We study formal languages which are capable of fully expressing quantitative
probabilistic reasoning and do-calculus reasoning for causal effects,
from a computational complexity perspective. 
We focus on satisfiability problems whose instance formulas allow expressing many 
tasks in probabilistic and causal inference.  
The main contribution of 
this work is establishing the exact computational complexity of these 
satisfiability problems. 
We introduce a new natural complexity class, 
named $\csuccETR$, which can be viewed as a succinct variant of the 
well-studied class $\exists\R$, and show that the problems we consider are complete 
for $\csuccETR$. Our results imply even stronger algorithmic limitations 
than were proven by Fagin, Halpern, and Megiddo (1990) and 
Moss\'{e}, Ibeling, and Icard (2022) for some variants of the standard
languages used commonly in probabilistic and causal inference.
\end{abstract}

\section{Introduction}
Satisfiability problems play a key role in a broad range of research fields, including AI, 
since many real-life tasks can be reduced in a natural and efficient way to instances 
of satisfiability expressed in a suitable formal language. A prominent example is the 
Boolean satisfiability problem ({\sc Sat}) whose instances represent Boolean formulas 
of propositional logic with the goal to determine whether there exists an assignment that 
satisfies a given formula. It is well known that any decision problem in the complexity 
class $\ccNP$ can be reduced in polynomial time to the {\sc Sat} problem,
see \cite{michael1979computersNP}. 

In this work, we investigate satisfiability problems (and their validity counterparts)
whose instance formulas allow expressing many tasks in probabilistic and causal inference. The formulas are specified 
in languages commonly used in pure probabilistic and interventional reasoning.
In particular, the probabilistic language, called in  
Pearl's Causal Hierarchy\footnote{See \cite{pearl2018book} for a first-time introduction to the topic.}
the \emph{associational} one, consists of Boolean combinations of 
(in)equalities involving pure probabilities as, e.g., 
$\PP{y \mmid x}$ which\footnote{In our paper, by 
$\PP{y\mmid x}$, $\PP{y\mmid\dop(x)}$,  etc., we mean $\PP{Y=y\mmid X=x}$, 
$\PP{Y=y\mmid\dop(X=x)}$, etc.} 
can ask  \emph{how likely is $Y=y$ given that observing $X=x$?}
An example formula in this language is the following single equality
\begin{align}\label{eq:factor}
& \sum_{u,x,z,y}(\PP{u,x,z,y} - \PP{u} \PP{x\mmid u} \PP{z\mmid x} \PP{y\mmid z,u})^2 = 0
 \end{align}
 which expresses the fact that the joint distribution can be factorized as
 $
   \PP{u,x,z,y} = \PP{u} \PP{x\mmid u} \PP{z\mmid x} \PP{y\mmid z,u},
$ 
for all values $u,x,z,y$.
The causal language extends the probabilistic language by allowing additionally to use terms involving 
Pearl's do-operator \cite{Pearl2009} as, e.g., $\PP{y\mmid\dop(x)}$
which can ask hypothetical questions such as
\emph{how likely would $Y = y$ be given an intervention setting $X$ to $x$?}
An example formula in this language is the single equality 
 \begin{equation}
  \mbox{$\PP{y\mmid\dop(x)} = \sum_{z} \PP{z\mmid x}\sum_{x'}\PP{y\mmid x', z}\PP{x'}$}
\label{eqn:front:door}
\end{equation}
which allows estimating the (total) causal effect of the intervention $X=x$ on outcome 
variable $Y=y$ via the prominent front-door adjustment  
\cite{pearl1995causal}. 
It is well known that this formalism enables inference of properties that are impossible to 
reason about from correlational data using a purely probabilistic framework
\cite{shpitser2008complete,bareinboim2022pearl}.

Similarly as in the classical {\sc Sat}, both languages allow natural, polynomial-time reductions of many tasks
in probabilistic and causal inference to the satisfiability (or validity) problems. 
Prominent examples can be establishing properties of 
\emph{structural causal models} \cite{glymour2014discovering,Pearl2009,koller2009probabilistic,Elwert2013}
which endow researchers with graphical structures (also called causal Bayesian networks) 
to encode conditional independences and causal assumptions as a directed acyclic graph (DAG).
For instance, the problem to decide, for a given DAG $\cG$ and variables $X,Y,Z$, 
whether in every structural causal model compatible with the structure $\cG$, the 
causal effect of $X$ on $Y$ can be inferred via front-door 
adjustment~\eqref{eqn:front:door} can be reduced in polynomial time  to the validity problem.
Indeed, for any $\cG,X,Y,Z$, one can compute a formula $\varphi$ 
whose size is bounded by a polynomial in the number of nodes of $\cG$, 
such that the front-door adjustment is applicable, if and only if, $\varphi$ is valid.
E.g., for the DAG~$\cG$:
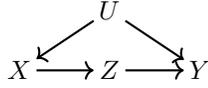
\begin{center}
\begin{tikzpicture}[xscale=1.5,yscale=1.0]
\tikzstyle{every node}=[inner sep=1pt,outer sep=1pt, minimum size=10pt,];
\tikzstyle{every edge}=[draw,->,thick]
\node (X) at (-0.8,0) {$X$};
\node (Z) at (0,0) {$Z$};
\node (Y) at (0.8,0,0) {$Y$};
\node (U) at (0.0,0.8) {$U$};
\draw (X) edge (Z); 
\draw (Z) edge (Y);
\draw (U) edge (X);
\draw (U) edge (Y);
\end{tikzpicture}
\end{center}
the formula $\varphi$ looks as follows $(\psi_1  \wedge \psi_2) \Rightarrow \varphi_{\textit{fda}}$, 
where $\varphi_{\textit{fda}}$ is the equality~\eqref{eqn:front:door} representing the front-door adjustment,
$\psi_1$ is the equality~\eqref{eq:factor} encoding the probability factorization, 
and $\psi_2$ is a conjunction of equalities, whose conjuncts encode edge orientations.
E.g., 
$X\to Z$ can be expressed as 
 \begin{equation} \label{eqn:front:orientation}
	  \mbox{$  \sum_{x,z} (\PP{x \mmid \dop(z)} - \PP{x})^2=0.$}
\end{equation}
The correctness of this reduction follows from the fact, that a
structural causal model over $X,Y,Z,U$ is supplemented with the above DAG  $\cG$ if 
$\psi_1  \wedge \psi_2$ encodes $\cG$. 
Interestingly, using do-calculus, one can prove that the front-door 
adjustment is applicable for the discussed instance which means that 
the formula $(\psi_1  \wedge \psi_2) \Rightarrow \varphi_{\textit{fda}}$ is valid.

\paragraph{Our Contribution.} 
Despite its importance, computational complexity aspects
of the satisfiability problems for the general probabilistic and 
causal languages
remain unexplored. In our work, we investigate 
these issues and establish, for the first time, the exact computational 
complexity of the satisfiability problem for probabilistic languages, 
denoted as $\SATprobpolysum$, 
and for the interventional level of the causal hierarchy, denoted as $\SATcausalpolysum$.
 They simultaneously indicate  
the complexity of the validity problems for these languages.

Our 
results are based on a novel extension of the well-studied 
complexity class $\exists\R$ which, loosely speaking, consists of problems 
that are polynomial time reducible to deciding whether a system of 
polynomial (in)equalities over real unknowns has a solution.
The class $\exists \mathbb{R}$ includes, in an obvious way, $\ccNP$
and, what is highly nontrivial, it is contained in $\ccPSPACE$
\cite{grigoriev1988solving,existentialTheoryOfRealsCanny1988some,existentialTheoryOfRealsSchaefer2009complexity}.
However, none of the inclusions  $\ccNP \subseteq \exists\R \subseteq \ccPSPACE$
are known to be strict.
Our new complexity class,  named $\csuccETR$,  can be 
viewed as a succinct variant of $\exists\R$ and it is an intermediate 
class between the exponential versions of $\ccNP$ and $\ccPSPACE$:
\begin{align}\label{eq:inclusions:succETR}
	&\ccNEXPTIME \subseteq \csuccETR \subseteq \ccEXPSPACE.
\end{align}
 The main contribution of this paper is to show that deciding the satisfiability 
of both probabilistic reasoning and  reasoning about causal interventions 
is complete for $\csuccETR$:  
\begin{theorem}[Main]\label{thm:main:result}
	The satisfiability problems $\SATprobpolysum$ and $\SATcausalpolysum$ are $\csuccETR$-complete.
 \end{theorem}
 
 Since $\csuccETR$ contains all problems decidable 
in non-deterministic exponential time, this shows a significant jump in complexity 
compared to the classic {\sc Sat} problem for propositional logic which is known 
to be $\ccNP$-complete.

\paragraph{The Existential Theory of the Reals.}
The Existential Theory of the Reals  (ETR) asks, given a Boolean combination of (in)equalities of multivariate polynomials, called a sentence, whether there exists an assignment of real numbers to the variables of the polynomials that satisfies the combination of (in)equalities?
The example
$$
\exists x\exists y\quad   x^2 -y =0 \ \wedge \ x^3 -x=0 \ \wedge \ (x> 0\  \vee\ y\le -1)
$$
illustrates  a \emph{yes} instance of ETR because $(x,y)=(1,1)$ is a solution of the sentence.
The theory forms its own complexity class $\exists \mathbb{R}$ which is 
defined as the closure of the ETR under polynomial time many-one reductions.
The importance of $\exists \mathbb{R}$ is underlined by the fact, that 
many  meaningful problems including algebraic, geometric, graph- and 
game-theory, machine learning, and continuous constraint satisfaction problems, 
have been classified as complete for $\exists \mathbb{R}$, see e.g.\
\cite{DBLP:journals/mst/SchaeferS17,abrahamsen2018art,garg2018r,abrahamsen2021training,miltzow2022classifying}.
Moreover, assuming $\ccNP\not= \exists \mathbb{R}$ which is widely believed,
the $\exists \mathbb{R}$-completeness of a problem reflects the limitation on the use
of algorithmic approaches to solve the problem, 
as, e.g., dynamic programming, divide and conquer, 
 tree-width based algorithms,  or {\sc Sat}-solver approaches
which are applicable for instances of $\ccNP$-complete problems.
The $\csuccETR$-completeness  of~$\SATprobpolysum$ and $\SATcausalpolysum$
(Theorem~\ref{thm:main:result}) 
imply even stronger algorithmic limitations than could be deduced 
from the $\exists \mathbb{R}$-completeness results of 
\cite{ibeling2022mosse} (discussed below) which 
concerns, however, some restricted variants of the languages.

 \paragraph{Previous Work on the Hardness of Satisfiability Problems.}
In the past decades, several research groups have studied the computational
complexity aspects of satisfiability problems for some limited languages of
probabilistic and causal inference.
In their pioneering work, \citeauthor{fagin1990logic}~[\citeyear{fagin1990logic}]
consider a language for 
reasoning about pure probabilities, which consists of Boolean combinations 
of \emph{linear} (in)equalities over probabilities, like 
$\PP{X=1\vee X=2}+2\PP{Y=1}\le \sfrac{1}{3}
\ \wedge\ \PP{Y=1}\ge \sfrac{2}{3}$.
The authors provide a complete axiomatization for the logic and show that the 
problem of deciding satisfiability
is NP-complete which, surprisingly, 
is no worse than that of propositional logic.
Later on, further studies explored  the complexity aspects of probability logics
 \cite{abadi1994decidability,speranski2013complexity} and reasoning about causal models
\cite{halpern2000axiomatizing,eiter2002complexity,aleksandrowicz2017computational}.

Most germane to our work are the recent studies of 
\citeauthor{ibeling2020probabilistic}~[\citeyear{ibeling2020probabilistic}]
and \citeauthor{ibeling2022mosse}~[\citeyear{ibeling2022mosse}]
which extended the formalism of \cite{fagin1990logic},
providing the more expressive languages 
for the second (\emph{interventional}) and third (\emph{counterfactual}) level of Pearl's Causal Hierarchy.
The~languages consist of Boolean combinations of \emph{polynomial} (in)equalities 
over pure probabilities or over probabilities involving the do-operator, respectively.
In particular, they allow to express \emph{conditioning}, as, e.g.,
$\PP{y\mmid x}=\PP{y, x}/\PP{x}$. 
But  the languages limit the  use  of \emph{marginalization} since 
the summation operator $\Sigma$ over the domain of random 
variables, as used, e.g., in Eq.~\eqref{eq:factor}--\eqref{eqn:front:orientation}, 
is not allowed.
Consequently,  to express the marginal distribution of  a variable $Y$ over 
a subset of variables $\{Z_1,\ldots,Z_m\}\subseteq \{X_1,\ldots, X_n\}$, as
$\sum_{z_1,\ldots,z_m} \PP{y,z_1,\ldots,z_m}$, in the language without summation
requires an extension 
$\PP{y,Z_1=0,\ldots,Z_m=0} + \ldots +\PP{y,Z_1=1,\ldots,Z_m=1}$
of exponential size in $m$. (In the expression, we assume that the $Z_i$ are  binary variables). 
Similarly,  $\PP{y \mmid  \dop(x)}  =  \sum_z \PP{y  \mmid x, z} \PP{z}$
is expressed 
equivalently as
$\PP{y \mmid  \dop(x)}  =  \PP{y  \mmid x, Z=0} \PP{Z=0} +  \PP{y  \mmid x, Z=1} \PP{Z=1}$.
While  this might be acceptable for one variable $Z$, the 
expansion of the sums grows exponentially when we have several variables.
Therefore, having an explicit summation operator is highly desirable.

\citeauthor{ibeling2020probabilistic}~[\citeyear{ibeling2020probabilistic}] and
 \citeauthor{ibeling2022mosse}~[\citeyear{ibeling2022mosse}]
give complete axiomatizations of these languages and investigate the 
computational complexity of the satisfiability problem for each language. 
In particular, they prove that,  although the languages for 
interventional and counterfactual inference 
are more expressive 
than the ones for probabilistic reasoning, the computational complexities of the
satisfiability problem for the logics
are both $\exists\R$-complete.
Combining these results with ours, we obtain 
a precise complexity theoretic classification of the satisfiability problems for 
all three levels of Pearl's Causal Hierarchy combined with 
various languages for terms  involving  probabilities
(proven by $^{\dagger}$: \citeauthor{fagin1990logic}~[\citeyear{fagin1990logic}],
$^{\ddagger}$: \citeauthor{ibeling2022mosse}~[\citeyear{ibeling2022mosse}]):
%


\begin{center}\footnotesize
\setlength\extrarowheight{3pt}
\begin{tabular}{|l|c|c|c|l|}
\hline
\emph{Terms} & \emph{prob.} & \emph{interven.} & \emph{counterfact.} & \emph{Source}\\[3pt]\hline
 \emph{lin} &  \multicolumn{3}{c|}{$\ccNP$-complete} &  $^{\dagger}$, $^{\ddagger}$\\[3pt]\hline
 \emph{poly} &  \multicolumn{3}{c|}{$\exists\R$-complete} &   $^{\ddagger}$\\[3pt]\hline
 \emph{poly} \& $\Sigma$ &  \multicolumn{2}{c|}{$\csuccETR$-complete} &$\csuccETR$-hard& Thm.~\ref{thm:main:result}\\[3pt] \hline 
\end{tabular}
\end{center}

While the languages of \citet{ibeling2022mosse} are capable of fully expressing quantitative
probabilistic reasoning, 
respectively,  do-calculus reasoning for causal effects, as discussed above, due to the lack of 
the summation operator $\Sigma$,  
they do not capture the standard notation used commonly in probabilistic and causal inference.
%
Consequently, to analyze the computational complexity aspects of the inference,
languages that exploit the standard notation need to be used. Indeed,
for the computational complexity, the key issue is how the instances are represented
since the time or space complexities are functions in the length of the input.
For example, if the language as in \citet{ibeling2022mosse} would be used 
(instead of the succinct one with summation operator), then the time complexity of many algorithms, 
including, e.g., the seminal 
\citeauthor{ShpitserIDCAlgorithm} [\citeyear{ShpitserIDCAlgorithm}]  algorithm to estimate the interventional 
distribution, would jump from polynomial to exponential.
Another problem would also be to construct polynomial time reductions to the satisfiability problems
since using the succinct encodings as, e.g., in 
Eq.~\eqref{eq:factor}--\eqref{eqn:front:orientation}, would not be allowed.

\paragraph{Structure of the Paper.}
The remainder of this work is dedicated to proving Theorem~\ref{thm:main:result}.
After providing preliminaries (Sec.~\ref{sec:lang} and \ref{sec:sat:problems}),
we introduce the class $\csuccETR$ in Sec.~\ref{sec:succETR} and provide 
the first complete problems for $\csuccETR$ in Sec.~\ref{sec:first:compl:problems}.
In Sec.~\ref{sec:completeness}, we prove 
the membership of $\SATprobpolysum$ and $\SATcausalpolysum$ 
in the class $\csuccETR$ and their $\csuccETR$-hardness.
The omitted proofs can be found in the appendix\IJCAIPROCEEDINGSONLY{ of  a companion paper \citep{zander2023ijcai-arxiv}}.

\section{Preliminaries}\label{sec:lang}

\paragraph{Syntax of Probabilistic and Causal Languages.}
The syntax 
used in this paper 
extends the definitions in  \cite{ibeling2020probabilistic,ibeling2022mosse}
for the languages of the first two levels of Pearl’s Causal Hierarchy. 
We consider discrete distributions and represent 
the values of the random variables as $\mathit{Val} = \{0,1,\myldots, \maxvaluecount - 1\}$ 
and denote by $\bX$ the set of all (endogenous) random variables.
By capital letters $X_1,X_2, \myldots$, we denote the individual variables 
and assume, w.l.o.g.,~that they all share the same domain  $\mathit{Val}$.
A value of $X_i$ is denoted by $x_i$ or a natural number.

In order to reason about the (in)equalities of arithmetic terms involving probability expressions, we need to define languages describing probabilities, languages describing arithmetic terms, and languages describing (in)equalities. 
The first languages characterize probabilistic and causal events as 
Boolean conditions over the values of 
(endogenous) random variables:
\begin{align*}
 \Lprop &::= X = x  
          \mid \neg \Lprop \mid \Lprop \wedge \Lprop 
          \\
\Lint &::= \top \mid  X = x  
     \mid  \Lint \wedge \Lint  
        \\
 \Lfull &::= [\Lint] \Lprop 
\end{align*}
where $X\in \bX$, and $x$ is either in $\mathit{Val}$ or is a summation variable as defined below.

To make the notation more 
convenient for our analysis, we use $[\Lint]$ 
in the syntax above to describe an \emph{intervention} on certain variables. 
The operator $[\cdot]\delta$ creates a new \emph{post-intervention} model
and the assigned propositional formula $\delta$
only applies to the new model.
Thus, for example, $\PP{[X=x]Y=y}$, which, using do-notation 
can be expressed as $\PP{Y=y\mmid \dop(X=x)}$,  
denotes the probability that the variable $Y$ equals $y$ in the model after the 
intervention $\compactEquals{\dop(X=x)}$.
Since $\top$ means that no intervention has been applied,
 we can assume that $\Lprop \subseteq \Lfull$.


Inserting the primitives 
into arithmetic expressions, we get 
the language 
$\Tpolysum(\cL)$, where, for $\cL\in \{\Lprop,\Lfull\}$, 
 $\delta \in\cL$,  and  $\delta'\in\Lprop$, 
 any $\be \in \Tpolysum(\cL)$ is formed by the grammar\footnote{This is a recursive definition where, e.g., $\be+\be'$ means, for $e,e'\in \Tpolysum(\cL)$, the expression $e+e'$ is also in $\Tpolysum(\cL)$. }:
\begin{equation}\label{eq:gramm:expr}
 \be::=  
 \PP{\delta\mmid \delta'}
 	\mid \be + \be' \mid \be  \cdot \be' \mid 
 \mbox{$\sum_{x}\be$.}
\end{equation}
The probabilities of the form $\PP{\delta}$  or $\PP{\delta\mmid \delta'}$, 
are called \emph{primitives} or \emph{(atomic) terms}.


In the summation operator $\sum_{x}$ in definition~\eqref{eq:gramm:expr}, we have
a dummy variable $x$ which ranges over all values $0,1, \myldots,\maxvaluecount - 1$.
The summation $\sum_{x} \be$ is a purely syntactical 
concept which represents the sum 
$\be[\sfrac{0}{ x}]  +\be[\sfrac{1}{x}]+\myldots +\be[\sfrac{\maxvaluecount - 1}{x}]$,
where, by $\be[\sfrac{v}{x}]$, we mean the expression in which all occurrences of $x$
are replaced with value $v$.
E.g., for  $\mathit{Val} = \{0,1\}$,
the expression\footnote{As usually, $\PP{Y=1, X=x}$, etc., means $\PP{Y=1 \wedge X=x}$, etc.}
$\sum_{x} \PP{Y=1, X=x}$
 semantically represents $\PP{Y=1, X=0} + \PP{Y=1, X=1}$.
%
We note that the dummy variable $x$ is not a (random) variable in the usual sense
and that its scope is defined in the standard way.


Finally, we define the languages of Boolean combinations of inequalities, following the grammar,
where  $ \be,\be'$ are expressions in  $ T(\Lprop) $ for \emph{prob} and $ T(\Lfull) $ for \emph{causal},
respectively:
\begin{align*}
    \Lprobpolysum&::= \be \le \be' \mid \neg \Lprobpolysum \mid \Lprobpolysum \wedge \Lprobpolysum \\
    \Lcausalpolysum&::= \be \le \be' \mid \neg \Lcausalpolysum \mid \Lcausalpolysum \wedge \Lcausalpolysum .
\end{align*}

%
%

Together with 
conjunctions, the comparisons yield an equality relation. Moreover,
for the atomic formula 
$\let\IsInPP=1 X=x$ in probabilities as, e.\,g. $\PP{X=x}$, we use a common 
abbreviation $\PP{x}$ if this does not lead to confusion. Thus, for example, 
Eq.~\eqref{eq:factor} builds the correct formula in the language 
$\Lprobpolysum$ and Eq.~\eqref{eqn:front:door} and \eqref{eqn:front:orientation} are
the correct formulas in  $\Lcausalpolysum$.

Although the language and its operations can appear rather restricted, all the usual elements of probabilistic and causal formulas can be encoded. Namely, equality is encoded as greater-or-equal in both directions, e.g. $\PP{x} = \PP{y}$ means $\PP{x} \geq \PP{y} \wedge \PP{y} \geq \PP{x}$.
Subtraction or divisions can be encoded by moving a term to the other side of the equation, e.g., 
$\PP{x} - \PP{y} = \PP{z}$ means $\PP{x} = \PP{z} + \PP{y}$, and $\PP{x} / \PP{y} = \PP{z}$ means $\PP{x} = \PP{z} \PP{y}$.
Any positive integer can be encoded from the fact $\PP{\top} \equiv 1$, e.g. $4 \equiv (1 + 1) (1 + 1) \equiv (\PP{\top} + \PP{\top}) (\PP{\top} + \PP{\top})$. The number~$0$ can be encoded as inconsistent probability,
i.e.,
$\PP{X=1 \wedge X=2}$. 
Note that these encodings barely change the size of the expressions, so allowing or disallowing these additional operators does not affect any complexity results involving these expressions.

\paragraph{Semantics.} 
\label{sec:mdels:norm}
We define a structural causal model (SCM) as in \cite[Sec.~3.2]{Pearl2009}.
An SCM 
is a tuple $\fM=(\cF, P, $ $\bU, \bX)$, such that $\bV = \bU \cup \bX$ is a set of 
variables partitioned into 
exogenous (unobserved) variables $\bU=\{U_1,U_2,\myldots \}$ and 
endogenous variables $\bX$.
The tuple $\cF=\{F_1,\myldots,F_n\}$ consists of
functions such that function $F_i$ calculates the value of variable $X_i$ from the values 
$(\bx, \bu)$ of other variables in $\bV$ as  
$F_i(x_{i_1},\myldots,x_{i_k},\fu_i)$ \footnote{We consider recursive models, 
that is, we assume the endogenous variables 
are ordered such that variable $X_i$ (i.e. function $F_i$) is not affected by any 
 $X_j$ with $j > i$.}. 
$P$ specifies a probability distribution 
of all exogenous  variables $\bU$. Since variables $\bX$  depend deterministically on 
 the exogenous variables via functions $F_i$,
 $\cF$ and $P$ define 
the obvious joint probability distribution of 
$\bX$.
 
 For any atomic  $\Lint$-formula $\let\IsInPP=1 X_i=x_i$ 
 (which, in our notation, means $\let\IsInPP=1 \dop(X_i=x_i)$),  we denote 
by $\cF_{X_i=x_i}$ the functions obtained from $\cF$ by replacing $F_i$
with the constant function $F_i(\bv):=x_i$. 
We generalize this definition for any interventions specified by
$\alpha\in \Lint$ in a natural way and denote as 
$\cF_{\alpha}$ the resulting functions.

For any $\varphi\in \Lprop$,  we write $\cF, \bu \models \varphi$
if $\varphi$ is satisfied for values of $\bX$ calculated from the values $\bu$.
For $[\alpha]\varphi \in \Lfull$, we write $\cF, \bu \models [\alpha]\varphi$ if  
$\cF_{\alpha}, \bu \models \varphi$.
  %
Finally, for $\psi\in \Lfull$, let $S_{\fM}(\psi)=\{\bu \mid \cF, \bu \models \psi\}$.
We assume, as is standard, the measurability of $\fM$ which guarantees 
that  $S_{\fM}(\psi)$ is always measurable. 

We define $\llbracket \be \rrbracket_{\fM} $
recursively in a natural way, 
starting with atomic terms as follows: $\llbracket \PP{\psi} \rrbracket_{\fM} = P(S_{\fM}(\psi))$,
resp.~$\llbracket \PP{[\alpha]\varphi\mmid \varphi'} \rrbracket_{\fM} = 
P(S_{\fM}([\alpha](\varphi\wedge \varphi'))) / P(S_{\fM}([\alpha]\varphi'))$.
For two expressions $\be_1$ and $\be_2$, we define 
$ \fM \models  \be_1 \le  \be_2$,  if and only if, 
$\llbracket \be_1 \rrbracket_{\fM}\le \llbracket \be_2 \rrbracket_{\fM}.$
The semantics for negation and conjunction are defined in the usual way,
giving the semantics for $\fM \models \varphi$ for any formula.

\paragraph{Complexity Notation.}\label{sec:prelim}
We use the well-known complexity classes
$\ccNP{}$, 
$\ccPSPACE{}$, 
$\ccNEXPTIME{}$, 
$\ccEXPSPACE{}$,  and $\cETR$ 
\cite{arora2009computational}.  
%
For two computational problems $A,B$, we will write $A\lep B$ if $A$ can be reduced to $B$ in polynomial time, which means $A$ is not harder to solve than $B$. A problem $A$ is complete for a complexity class $\complexityclass{C}$, if $A \in \complexityclass{C}$ and, for every other problem $B\in\complexityclass{C}$, it holds $B\lep A$. By ${\tt co}\mbox{-}\complexityclass{C}$ we denote the class of all problems
$A$ such that its complements $\overline{A}$ belong to $\complexityclass{C}$.

\section{Satisfiability Problems}\label{sec:sat:problems}
The decision problems $\SATcausalpolysum$ and $\SATprobpolysum$
takes as input a formula $\varphi$ 
in  
language $\Lcausalpolysum$
and in $\Lprobpolysum$, respectively,
and  asks whether there exists a model 
$\fM$ such that $\fM \models\varphi$.
Analogously, we define the validity problems for languages 
$\Lcausalpolysum$ and $\Lprobpolysum$ 
of deciding whether, for a given $\varphi$,
$\fM \models\varphi$ holds
for all models  $\fM$.
%
%
From the definitions, it is obvious that \emph{causal} variants of the problems 
are at least as hard as the \emph{prob}  counterparts. 
%

To give a first intuition about the expressive power of the $\SATprobpolysum$  problem,
we prove the following result which shows that 
$\SATprobpolysum$ can encode any problem in $\ccNEXPTIME{}$ efficiently.
\begin{proposition}\label{thm:exist:equal:nexptime:hard}
The $\SATprobpolysum$ problem
is $\ccNEXPTIME{}$-hard. 
\end{proposition}

The remaining part of the paper is devoted to proving the main result (Theorem~\ref{thm:main:result})
showing that the satisfiability problems are complete for the new  class 
$\csuccETR$. Since a formula $\varphi$ is \emph{not} valid, if and only if $\neg\varphi$ 
is satisfiable and $\neg\varphi$ is in $\Lprobpolysum$ or in $\Lcausalpolysum$, respectively,
we can conclude from Theorem~\ref{thm:main:result}:
\begin{corollary}\label{corr:main}
The validity problems for the languages $\Lprobpolysum$ and
$\Lcausalpolysum$  are complete for ${\tt co}\mbox{-}\csuccETR$,
which is related to standard classes as follows:
	${\tt co}\mbox{-}\ccNEXPTIME \subseteq {\tt co}\mbox{-}\csuccETR \subseteq \ccEXPSPACE$\,.
\end{corollary}

%

\section{The Existential Theory of the Reals, Succinctly}\label{sec:succETR}
The existential theory of the reals $\ETR$ is the set of true sentences
of the form
\begin{equation} \label{eq:etr:1}
   \exists x_1 \dots \exists x_n \varphi(x_1,\dots,x_n).
\end{equation}
$\varphi$ is a quantifier-free Boolean formula over the basis $\{\vee, \wedge, \neg\}$
and a signature consisting of the constants $0$ and $1$, the functional symbols
$+$ and $\cdot$, and the relational symbols $<$, $\le$, and $=$. The sentence
is interpreted over the real numbers in the standard way.


All operations have an arity of at most two, so we can represent $\varphi$ as a binary
tree. The leaves of the tree are labeled with variables or constants,
the inner nodes are labeled with $+$, $\cdot$, $=$, $<$, $\le$, $\wedge$, $\vee$,
or $\neg$. $\neg$-nodes have in-degree one, and all other inner nodes have in-degree two.

We now define succinct encodings of instances for $\ETR$. Succinct encodings
have been studied for problems in $\ccNP$, see e.g.\ \cite{papadimitriou}.  
The input is now a Boolean circuit $C$ computing a function $\{0,1\}^N \to \{0,1\}^M$,
which encodes the input instance. For example, an instance of the well-known 3-{\sc Sat} problem
is some standard encoding of a given 3-CNF formula. In the succinct version of 3-{\sc Sat},
we are given a circuit $C$. $C$ encodes a 3-CNF formula $\psi$ in the following
way: $C(i)$ is an encoding of the $i$th clause of $\psi$.
In this way, one can encode an exponentially large 3-{\sc Sat} formula $\psi$ 
(with up to $2^N$ clauses) .
The succinct encoding typically induces a complexity jump:
The succinct version of the 3-{\sc Sat} problem is $\ccNEXPTIME$-complete. 
The same is true for many other $\ccNP$-complete problems \cite{papadimitriou}.

In the case of $\ETR$, the formula $\varphi$ 
is now given succinctly by a Boolean circuit $C$. The circuit $C$
computes a function $\{0,1\}^N \to \{0,1\}^M$. The input
is (an encoding of) a node~$v$ of the formula and the output $C(v)$ contains the label
of the node as well as its parent and its children (encoded as a string in binary).
In this way, we can represent a formula that is exponentially large
in the size of the circuit. 

\begin{example}
Consider the formula $\exists x_1: x_1^2 = 1 + 1$ (``$2$ has a square root over the reals''). 
As an instance
of $\ETR$, this formula would be encoded as a binary string 
in some standard way. 
Figure~\ref{fig:succETR:ex} shows the binary tree underlying the formula. 
There are seven nodes, which we can represent as integers $1,\ldots, 7$. 
As a $\succETR$ instance, the formula will be given by 
a circuit $C$, computing a function 
$\{1,\ldots,7\}\to \{=,+,\cdot,0,1,x_1\} \times \{0, 1,\ldots,7\}^3$
describing the tree locally. For instance, $C(1) = (=,0,2,3)$,
means the following:  the label of the node $1$ is $=$. It
has no parent, that is, it is the root, which is indicated by giving $0$ as the parent.
The left child is $2$ and the right child is $3$. 
In the same way $C(4) = (x_1,2,0,0)$,
which means that the node $4$ is a leaf with parent $2$ and  it is labeled with the variable $x_1$.
We can represent the integers $\le 7$ as bit strings of length $N=3$
and the tuples $C(i)$,  e.g., $(=,0,2,3)$ or $(x_1,2,0,0)$,
can be encoded as bit strings of some length $M$.
Thus, $C$ can be considered a function $C:\{0,1\}^N \to \{0,1\}^M$.
Note that we cannot use more variables than the formula has nodes, thus
to encode the tuple $C(i)$, we need at most $M\le (3+N) + 3N$ bits: 
in the first entry, we need $3 + N$ bits to encode a label 
$+$, $\cdot$, $=$, $<$, $\le$, $\wedge$, $\vee$, and $\neg$ or 
an index $j$ of a variable $x_j$; to encode each of the remaining components,
$N$ bits suffice.
It is clear that one can construct such a circuit $C$. When the formula is large but also structured, then the circuit $C$ can be much smaller than the formula itself. 
\end{example}

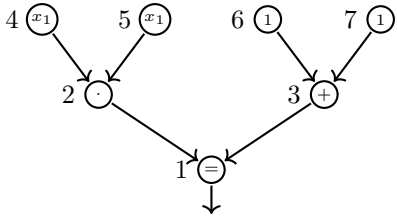
\begin{figure} 
\begin{center}
\begin{tikzpicture}[xscale=1.5,yscale=1.0]
\tikzstyle{every node}=[circle, draw=black, thick, fill=white, inner sep=1pt,outer sep=1pt, minimum size=10pt,];
\tikzstyle{every edge}=[draw,->,thick]
\node (u1) at (-0.5,0) {\tiny$x_1$};
  \node[xshift=-.4cm,fill=none,draw=none] at (-0.5,0) {$4$};
\node (u2) at (0.5,0) {\tiny$x_1$};
  \node[xshift=-.4cm,fill=none,draw=none] at (0.5,0) {$5$};
\node (u3) at (2.5,0) {\tiny$1$};
  \node[xshift=-.4cm,fill=none,draw=none] at (2.5,0) {$7$};
\node (u4) at (1.5,0) {\tiny$1$};
  \node[xshift=-.4cm,fill=none,draw=none] at (1.5,0) {$6$};  
\node (v) at (0,-1) {\tiny$\cdot$};
  \node[xshift=-.4cm,fill=none,draw=none] at (0,-1) {$2$};
\node (w) at (2,-1) {\tiny$+$};
  \node[xshift=-.4cm,fill=none,draw=none] at (2,-1) {$3$};
\node  (o) at (1,-2) {\tiny$=$};
  \node[fill=none,draw=none]  (out) at (1,-2.8) {$ $};
  \node[xshift=-.4cm,fill=none,draw=none] at (1,-2) {$1$};
\draw (u1) edge (v);
\draw (u2) edge (v);
\draw (u4) edge (w);
\draw (u3) edge (w);
\draw (v) edge (o);
\draw (w) edge (o);
\draw (o) edge (out);
\end{tikzpicture}
\vspace{-5mm}
\end{center}
\caption{A tree representing the $\ETR$ instance $\exists x_1: x_1^2 = 1 + 1$.
The nodes set is $\{1,\dots,7\}$, which we can represent by bit strings of length $3$.
The labels of the nodes are drawn in the circles.}
\label{fig:succETR:ex}
\end{figure}

\begin{definition}[$\succETR$]
$\succETR$ is the set of all Boolean circuits $C$ that encode a true sentence $\varphi$
as in $(\ref{eq:etr:1})$ as follows. Assume that $C$ computes a function 
$\{0,1\}^N \to \{0,1\}^M$. Then $\{0,1\}^N$ is the node set of the tree underlying $\varphi$
and $C(i)$ is the encoding of the description of node $i$, consisting of the label
of $i$, its parent, and its two children. The variables in the formula
are $x_1,\dots,x_{2^N}$. 
\end{definition}

The number of variables is formally always a power of $2$. If we need fewer variables,
we simply do not use them in the formula. We still quantify over all of them,
but this is no problem. In the same way, the number of nodes is always
a power of $2$. If we need fewer nodes, we can define a special value for $C(i)$,
which marks node $i$ as unused.

We will always assume that all negations are pushed to the lowest level 
of the formula and then they are absorbed in the arithmetic terms, i.e.,
$\neg(s = t)$ becomes $(s < t) \vee (t < s)$, $\neg (s < t)$ becomes
$t \le s$ and $\neg (s \le t)$ becomes $t < s$. Furthermore, we can remove
the comparisons $\le$, since we can replace $s \le t$ by $(s < t) \vee (s = t)$.  

\begin{remark} \label{remark:pushdown}
In the non-succinct case, eliminating the negations by pushing them to the bottom
is no problem, as it can be done explicitly.
In the succinct case, it seems that we need to assume
that already the given input has this property.
\end{remark}

Finally, based on the $\succETR$ problem, we define the succinct version of $\exists\R$
as follows:
\begin{definition}
$\csuccETR$ is the class of all problems that are polynomial time reducible
to $\succETR$.
\end{definition}

$\csuccETR$ is related to known complexity classes as follows.

\begin{theorem}\label{thm:NEXPTIME:sub:succETR:sub:EXPSPACE}
$\ccNEXPTIME \subseteq \csuccETR \subseteq \ccEXPSPACE$.
\end{theorem}

\section{First Complete Problems for $\csuccETR$}
\label{sec:first:compl:problems}
In this section, we present the first problems which are complete for $\csuccETR$.
The key role plays the problem  $\sigmaETR$ which we need to prove 
the membership of $\SATcausalpolysum$ in the class $\csuccETR$
in Proposition~\ref{prop:SAT:causal:in:sigmaETR}.

$\QUAD$ is the problem to decide whether a given family
of quadratic multivariate polynomials has a common root. Here the polynomials are given as lists
of monomials. The coefficient is represented as a quotient of two integers.
$\QUAD$ is a complete problem for the class $\cETR$
of all problems that are reducible to $\ETR$ \cite{DBLP:journals/mst/SchaeferS17}.

The succinct version $\succQUAD$ 
is defined as follows: 
We get a Boolean circuit $C$ as input computing a function 
$\{0,1\}^K \times \{0,1\}^N \to \{0,1\}^M$. $C(x,y)$ is an encoding
of the $y$th monomial of the $x$th equation.

We allow that a monomial appears several times in these lists
and the coefficients all add up, so the final polynomials are
the sum of all monomials in the list.

\begin{lemma} \label{lem:succQUAD-hard}
$\succETR \lep \succQUAD$. Furthermore, the reduction only creates quadratic polynomials
with at most four monomials and coefficients $\pm 1$.
\end{lemma}



$\sigmaETR$ is defined like $\ETR$, but we add to the signature an additional summation operator. 
This is a unary operator $\sum_{x_j = a}^b$.
Consider an arithmetic term given by a tree with the top gate $\sum_{x_j = a}^b$.
Let $t(x_1,\dots,x_n)$ be the term computed at the child of the top gate. Then
the new term computes
$
   \sum_{e = a}^b t(x_1,\dots,x_{j-1},e,x_{j+1},\dots,x_n),
$
that is, we replace the variable $x_j$ by a summation variable~$e$, which then
runs from $a$ to $b$.

Furthermore, we allow a new form of variable indexing in $\sigmaETR$,
namely variables of the form $x_{n(x_{j_1},\dots,x_{j_m})}$:
This can only be used when variables $x_{j_1},\dots,x_{j_m}$ occur in the scope
of summation operators and are
replaced by summation variables $e_1,\dots,e_m$ with summation range $\{0,1\}$.
$x_{n(x_{j_1},\dots,x_{j_m})}$ is interpreted as the variable
with index given by $e_1,\dots,e_m$ interpreted as a number in binary. 

Instances of $\sigmaETR$ are not given succinctly by a circuit. However, the ability to
write down exponential sums succinctly makes the problem as hard as $\succETR$. We get

\begin{lemma}\label{lemma:seq:of:reductions}
	$\succQUAD \lep \sigmaETR 
	\lep \succETR$
\end{lemma}
which yields:
\begin{theorem}\label{thm:three:complete:problems}
The problems $\succQUAD$ and $\sigmaETR$ are $\csuccETR$-complete.
\end{theorem}


\section{Proof of the Main Theorem} 
\label{sec:completeness}
Equipped with the tools provided in the previous sections, 
we are ready to prove our main result (Theorem~\ref{thm:main:result})
by showing:
\begin{align}
& \SATprobpolysum, \SATcausalpolysum \in \csuccETR \  \ \mbox{and}  \label{eq:memb}\\
& \SATprobpolysum, \SATcausalpolysum \ \mbox{are $\csuccETR$-hard}. \label{eq:hard}
\end{align}
We first show the membership of $\SATprobpolysum$ in the class $\csuccETR$ 
proving the following reduction:
\begin{proposition}\label{prop:SAT:causal:in:sigmaETR}
  $\SATcausalpolysum \lep \sigmaETR$.  
\end{proposition}
Due to the transitivity of $\lep$ and  by
Theorem~\ref{thm:three:complete:problems},
we get that $\SATcausalpolysum \lep \succETR$. Moreover, since 
$\SATprobpolysum \lep \SATcausalpolysum $,
 \eqref{eq:memb} follows.

For~\eqref{eq:hard}, it is sufficient to show the 
$\csuccETR$-hardness of  $\SATprobpolysum$. 
We prove this in Proposition~\ref{prop:red:to:satprob}.
To this aim, we introduce the problem $\succETRcR{\sfrac{1}{8}}{-\sfrac{1}{8},{\sfrac{1}{8}}}$
and 
we prove the $\csuccETR$-hardness of this problem.

\begin{definition}[\citeauthor{abrahamsen2017artTR}, \citeyear{abrahamsen2017artTR}]
 In the problem $\ETRc{c}$, where $c \in \Rset$, we are given a set of real variables 
 $\{x_1,\myldots, x_n\}$ and a set of equations of the form $x_i=c, \ x_{i_1} + x_{i_2} = x_{i_3}, \ x_{i_1} x_{i_2}=x_{i_3},$
 for $i,i_1,i_2,i_3 \in [n]$. The goal is to decide whether the system of equations has a solution.
The problem, where we also require that $x_1,\myldots, x_n \in [a, b]$, for some
$a,b \in \Rset$, is denoted by  $\ETRcR{c}{a,b}$.
\end{definition}

\begin{lemma}[\citeauthor{abrahamsen2017artTR}, \citeyear{abrahamsen2017artTR}]
  $\ETRc{1}$ and $\ETRcR{\sfrac{1}{8}}{-\sfrac{1}{8},{\sfrac{1}{8}}}$ are $\cETR$-complete. 
\end{lemma}

%
%
%

Let $\succETRc{c}$ and $\succETRcR{c}{a,b}$ denote the succinct  versions of 
$\ETRc{c}$ and $\ETRcR{c}{a,b}$, respectively. We assume that their instances are 
represented as seven Boolean circuits 
$C_0,C_1,\myldots,C_6: \{0,1\}^M \to \{0,1\}^N$
such that 
$C_0(j)$ gives the index of the variables in the $j$th equation of type $x_i=\sfrac{1}{8}$,
$C_1(j),C_2(j),C_3(j)$ give the indices of variables in the $j$th equation of the type $x_{i_1}+x_{i_2}= x_{i_3}$, and 
$C_4(j),C_5(j),C_6(j)$ give the indices of variables in the $j$th equation of the type $x_{i_1} x_{i_2}= x_{i_3}$.
Without loss of generality, we can assume that an instance has the same number $2^M$ of 
equations of each type; If not, one of the equations can be duplicated as many times as needed.


\begin{lemma}\label{lemma:interv:succ:complete}
   $\succETRcR{\sfrac{1}{8}}{-\sfrac{1}{8},{\sfrac{1}{8}}}$ is $\csuccETR$-complete. 
\end{lemma}

To establish the lemma,
we first show that $\succQUAD \lep \succETRc{1}$ 
and then prove that 
$\succETRc{1}\lep  \succETRcR{\sfrac{1}{8}}{-\sfrac{1}{8},{\sfrac{1}{8}}}$
(see 
 \citep{zander2023ijcai-arxiv}).

%
%


Now we are ready to prove the $\csuccETR$-hardness of  $\SATprobpolysum$ through the following reduction:

\begin{proposition}\label{prop:red:to:satprob}
  $\succETRcR{\sfrac{1}{8}}{-\sfrac{1}{8},{\sfrac{1}{8}}} \lep \SATprobpolysum$.
\end{proposition}

\begin{proof}
Let us assume that the instance of $\succETRcR{\sfrac{1}{8}}{-\sfrac{1}{8},{\sfrac{1}{8}}} $ is
represented by seven Boolean circuits $C_0,C_1,\myldots,C_6: \{0,1\}^M \to \{0,1\}^N$ as described above.
Let  the variables of the instance 
be indexed as $x_{e_1,\ldots,e_N}$, with $e_i\in \{0,1\}$ for $i\in[N]$. 
Below, we often identify the bit sequence $b_1,\myldots, b_L$ by an integer $j$, with $0\le j \le 2^L -1$,
the binary representation of which is $b_1\myldots b_L$ and vice versa. 

The instance of the problem $\succETRcR{\sfrac{1}{8}}{-\sfrac{1}{8},{\sfrac{1}{8}}} $ is satisfiable if and only if: 
\begin{align}\label{eq:sat:cond:etr}
&\hspace*{-10mm} \exists x_0,\myldots, x_{2^N-1} \in [-\sfrac{1}{8},\sfrac{1}{8}]    \nonumber \\
     \sum_{j=0}^{2^M -1} \Big( &(x_{C_0(j)} - \sfrac{1}{8})^2 +
    (x_{C_1(j)} + x_{C_2(j)} - x_{C_3(j)})^2 +\nonumber \\
  & \ \  (x_{C_4(j)} \cdot  x_{C_5(j)} - x_{C_6(j)})^2 \Big) \ =  \ 0.
\end{align}
We construct a system of equations in the language $\Lprobpolysum$ and prove that 
there exists a model which satisfies the equations if and only if the formula~\eqref{eq:sat:cond:etr}
is satisfiable.

Let $X_0, X_1,\myldots,X_N$ be binary random variables. We will 
model each real variable 
$x_{e_1,\ldots,e_N}$ as a term involving the conditional probability  $\PP{X_0=0 \mid X_1 = e_1,\myldots, X_N = e_N} $ 
as follows:
$$
 \mbox{$ q_{e_1,\ldots, e_N}\  := \   \sfrac{2}{8} \cdot \PP{X_0=0 \mid X_1 = e_1,\myldots, X_N = e_N}  - \sfrac{1}{8}.$}
$$
This guarantees that 
$
   \mbox{$q_{e_1,\ldots, e_N} \in [ -\sfrac{1}{8} , \sfrac{1}{8}]$.}
$
In our construction, the existential quantifiers in the formula~\eqref{eq:sat:cond:etr} correspond to the 
existence of a probability distribution $P(X_0, X_1,\myldots,X_N)$ which determines the values $ q_{e_1,\ldots, e_N}$.
This means that the formula~\eqref{eq:sat:cond:etr} is satisfiable if and only if there exists a model for the equation:
\begin{align}\label{eq:sat:cond:etr:prob}
    \sum_{j=0}^{2^M -1} \Big( & (q_{C_0(j)} - \sfrac{1}{8})^2 +
    (q_{C_1(j)} + q_{C_2(j)} - q_{C_3(j)})^2 +\nonumber \\
  &  \ \ (q_{C_4(j)} \cdot  q_{C_5(j)} - q_{C_6(j)})^2 \Big) \ =  \ 0.
\end{align}

The challenging task which remains to be solved is to express the terms $q_{C_i(j)}$
in the language $\Lprobpolysum$. Below we provide a system of equations that achieves this goal.

To model a Boolean formula encoded by a node of $C_i$, with $i=0,1,\myldots, 6$, we use arithmetization 
to go from logical formulas to polynomials over terms involving probabilities of the form $\PP{\delta}$.
We start with input nodes $v\in\{0,1\}$ and model them as events $\delta_v$ such that $\PP{\delta_v}=1$ 
if and only if $v=1$. For every internal node $v$ of $C_i$, we proceed as follows.
If $v$ is labeled with $\neg$ and $u$ is a child of $v$, then we define an event $\delta_v$ such 
that $\PP{\delta_v}=1-\PP{\delta_u}$. If $v$ is labeled with $\wedge$ and $u$ and $w$ are children  
of $v$, then we specify an event $\delta_v$ such that $\PP{\delta_v}=\PP{\delta_u}\PP{\delta_w}$.
Finally, if $v$ is labeled with $\vee$ and $u$ and $w$ are children  
of $v$, then $\PP{\delta_v}=1-(1-\PP{\delta_u})(1-\PP{\delta_w})$.
Thus, if $v$ is an output node of a circuit $C_i$, then, 
for  $C_i$ fed with input $j=b_{1}\myldots b_{M}\in \{0,1\}^M$, we have $v=1$ if and only if
$\PP{\delta_v}=1$. We define the events as follows.

Let $u_{i,1},\myldots, u_{i,M}$ be the input nodes of $C_i$ and let, for every node $v$ of $C_i$, 
the sequence ${\ell^v_{1}}, \myldots, {\ell^v_{k_v}}$ denotes the indices of  the leaves 
$u_{i,\ell^v_{1}}, \myldots, u_{i,\ell^v_{k_v}}$ of the sub-circuit with root $v$.
Note that, for an input $u_{i,k}$, the index $\ell^{u_{i,k}}_{1}=k$ and $k_{u_{i,k}}=1$.
To define the events,  for every node $v$, we introduce a sequence of $M$ new binary random 
variables $X_{v,1}\myldots, X_{v,M}$ and provide an equation that enforces the properties described 
above. The probabilities used in the equations involve, for a node $v$, only the variables 
indexed with $\ell^v_{1}, \myldots, \ell^v_{k_v}$; The remaining variables are irrelevant. 

For the $k$th input node $u_{i,k}$ 
and the variable $X_{u_{i,k},k}$,  we require
  \begin{align} \label{eq:constr:1}
    &     \PP{X_{u_{i,k},k} =0}=0 \ \mbox{and}\  \PP{X_{u_{i,k},k} =1}=1 .
  \end{align}
 For every internal node $v$ of $C_i$, if $v$ is labeled with $\neg$ and $u$ is a child of $v$, 
 then we define the equation:
  \begin{align}\label{eq:constr:2}
    \sum_{{b_1},\ldots,b_{k_v}} & \big( 
   \PP{X_{v,\ell^v_1}=b_{1},\myldots,X_{v,\ell^v_{k_v}}=b_{k_v}} -  \nonumber \\
 &\ \  ( 1 -   \PP{X_{u,\ell^v_1}=b_1,\myldots,X_{u,\ell^v_{k_v}}=b_{k_v}})\big)^2 
 \  = \ 0. 
 \end{align}
 If $v$ has label $\wedge$ and $w$ and $z$ are  children  of $v$, 
 then we define 
  \begin{align}\label{eq:constr:3}
  \sum_{b_{1},...,b_{k_v}} 
  & \big(\PP{X_{v,\ell^v_1}=b_{1},\myldots,X_{v,\ell^v_{k_v}}=b_{k_v}}  \ -  \nonumber \\
  & \ \ \PP{X_{w,\ell^w_1}=b_{j_1},\myldots,X_{w,\ell^w_{k_w}}=b_{j_{k_w}}}\  \times \nonumber\\
 &  \ \    \PP{X_{z,\ell^z_1}=b_{j'_1},\myldots,X_{z,\ell^z_{k_z}}=b_{j'_{k_z}}} \big)^2 
 \ = \ 0 , 
 \end{align}
 where the indices $j_k,j'_{k'}\in \{ 1,\myldots k_v\}$ are defined as 
 $j_k:=i$ for $\ell^w_{k}=\ell^v_i$ and $j'_{k'}:=i$ for $\ell^z_{k'}=\ell^v_i$.
%
%

If $v$ is labeled with $\vee$, we encode it as a negation of $\wedge$ after negating the children of $v$.

Let $\hat{u}_{i,1},\myldots, \hat{u}_{i,N}$ be the  output nodes of $C_i$.
For the $k$th output $v=\hat{u}_{i,k}$ and for an $C_i$'s  input $j=b_{1}\myldots b_{M}$, we denote by 
$o_{i,k}(j)$ the expression
$$
   o_{i,k}(j) := \PP{X_{v,\ell^v_1}=b_{\ell^v_1},\myldots,X_{v,\ell^v_{k_v}}=b_{\ell^v_{k_v}}} 
$$ 
which represents the value  of the $k$th output bit of $C_i(j)$ in such a way that the bit is equal to $e_k$ 
if and only if  the probability $o_{i,k}(j)=e_k$. We illustrate this concept in Example~\ref{ex:out:prob}.



Finally, for $i=0,1,\myldots,6$ and $j=b_{1}\myldots b_{M}\in \{0,1\}^M$, let
\begin{align*}
 & \alpha(i,j):=
  \sum_{e_1,\dots,e_N \in \{0,1\}}  q_{e_1,\dots,e_N} \times \\
 &  \prod_{k = 1}^N (o_{i,k}(j) \PP{E=e_k} + (1 - o_{i,k}(j))(1 - \PP{E=e_k})),
\end{align*}
where 
$E$ is a new binary random variable with $\PP{E=0}=0$ and  $\PP{E=1}=1$.
The crucial property of this encoding is that:
For all $i=0,1,\myldots,6$ and $j=b_{1}\myldots b_{M}\in \{0,1\}^M$, it is true that 
$
  \alpha(i,j) = q_{C_i(j)}.
$

Now, we replace every term $q_{C_i(j)}$ in the Eq.~\eqref{eq:sat:cond:etr:prob} by $\alpha(i,j)$
and get the following final equation in the language $\Lprobpolysum$:
 \begin{align}\label{eq:sat:cond:prob:final}
     \sum_{j=0}^{2^M -1} \Big( &(\alpha(0,j) - \sfrac{1}{8})^2 +
    (\alpha(1,j)+ \alpha(2,j) - \alpha(3,j))^2 \ +\nonumber\\
&\ \    (\alpha(4,j) \cdot  \alpha(5,j) - \alpha(6,j))^2 \Big) \ =  \ 0.
\end{align}
This completes the proof since it is true that the formula~\eqref{eq:sat:cond:etr}
is satisfiable if and only if there exists a model which satisfies  Eq.~\eqref{eq:constr:1}--\eqref{eq:constr:3}
and Eq.~\eqref{eq:sat:cond:prob:final}.
Obviously, the size of the resulting system of equations is polynomial in the size $|C_0|+|C_1|+\myldots +|C_6|$ 
of the input instance and the system can be computed in polynomial time.
\end{proof}

\begin{example}\label{ex:out:prob}
Consider a Boolean circuit $C$ with three input nodes $u_1,u_2,u_3$, internal nodes $v,w$,  
and one output node $\hat u$ as shown in Fig.~\ref{fig:circ:ex}.
The relevant variables used to encode the output bit are:
$X_{u_1,1},X_{u_2,2},X_{u_3,3}$, assigned to the input nodes, as well as $X_{v,1},X_{v,2}$ and 
$X_{w,2},X_{w,3}$ and $X_{\hat{u},1}, X_{\hat{u},2}, X_{\hat{u},3}$, assigned to $v$, to $w$, and $\hat u$,
respectively. 
Then, assuming the constraints expressed in Eq.~\eqref{eq:constr:1}--\eqref{eq:constr:3} are satisfied, 
for any $b_1,b_2,b_3\in \{0,1\}$,  the probability $o(j=b_1b_2b_3) := \  \PP{X_{\hat{u},1}=b_1, X_{\hat{u},2}=b_2, X_{\hat{u},3}=b_3} $ 
can be estimated as follows: 
\begin{align*}
  o(j) & = \PP{X_{\hat{u},1}=b_1, X_{\hat{u},2}=b_2, X_{\hat{u},3}=b_3}  \\
       & = \PP{X_{v,1}=b_1, X_{v,2}=b_2} \cdot  \PP{X_{w,2}=b_2, X_{w,3}=b_3}  \\
       & = \PP{X_{u_1,1}=b_1}\cdot \PP{X_{u_2,2}=b_2}\cdot   \\
       &
      \quad \quad  \big(1 - (1- \PP{X_{u_2,2}=b_2})\cdot (1- \PP{X_{u_3,3}=b_3})) \big)  \\
       & = b_1 \cdot b_2 \cdot (1-(1-b_2) (1-b_3)),
 \end{align*}
which is $0$ if $C(b_1,b_2,b_3)=(b_1\wedge b_2)\wedge (b_2 \vee b_3)=0$, and it is 
equal to $1$ if $C(b_1,b_2,b_3)=1$.
\end{example}

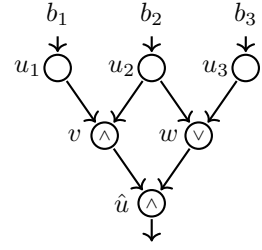
\begin{figure}
\begin{center}
\begin{tikzpicture}[xscale=1.25,yscale=0.9]
\tikzstyle{every node}=[circle, draw=black, thick, fill=white, inner sep=1pt,outer sep=1pt, minimum size=10pt,];
\tikzstyle{every edge}=[draw,->,thick]
\node (u1) at (0,0) {};
  \node[fill=none,draw=none]  (b1) at (0,0.8) {$b_1$};
  \node[xshift=-.4cm,fill=none,draw=none] at (0,0) {$u_1$};
\node (u2) at (1,0) {};
  \node[fill=none,draw=none]  (b2) at (1,0.8) {$b_2$};
  \node[xshift=-.4cm,fill=none,draw=none] at (1,0) {$u_2$};
\node (u3) at (2,0) {};
  \node[fill=none,draw=none]  (b3) at (2,0.8) {$b_3$};
  \node[xshift=-.4cm,fill=none,draw=none] at (2,0) {$u_3$};
\node (v) at (0.5,-1) {\tiny$\wedge$};
  \node[xshift=-.4cm,fill=none,draw=none] at (0.5,-1) {$v$};
\node (w) at (1.5,-1) {\tiny$\vee$};
  \node[xshift=-.4cm,fill=none,draw=none] at (1.5,-1) {$w$};
\node  (o) at (1,-2) {\tiny$\wedge$};
  \node[fill=none,draw=none]  (out) at (1,-2.8) {$ $};
  \node[xshift=-.4cm,fill=none,draw=none] at (1,-2) {$\hat u$};
\draw (u1) edge (v);
\draw (u2) edge (v);
\draw (u2) edge (w);
\draw (u3) edge (w);
\draw (v) edge (o);
\draw (w) edge (o);
\draw (b1) edge (u1);\draw (b2) edge (u2);\draw (b3) edge (u3);
\draw (o) edge (out);
\end{tikzpicture}
\vspace{-6mm}
\end{center}
\caption{An example Boolean circuit with  input nodes $u_1,u_2,u_3$, internal nodes $v,w$,  
and one output node $\hat u$.}\label{fig:circ:ex}
\end{figure}

\section{Discussion}
We have analyzed the complexity of deciding whether a system of (in)equalities
involving probabilistic and causal formulas is satisfiable using standard languages,
allowing formulas with summations. We have shown that the problems are complete 
for a new complexity class $\csuccETR$. Using these results, we could conclude
that the complexity of the validity problem, asking whether a Boolean combination of 
(in)equalities involving pure probabilities and interventional expressions is valid,
is complete for the class ${\tt co}\mbox{-}\csuccETR$.

%

Any $\exists\R$-complete problem can probably be turned into a $\csuccETR$-complete one by encoding the input succinctly. Although we do not know other published $\csuccETR$-complete problems, the situation is similar to the relation of $\ccNP$ and $\ccNEXPTIME$. The succinct versions of almost all $\ccNP$-complete problems are $\ccNEXPTIME$-complete.
We think that $\csuccETR$ is so natural that further complete problems not coming from $\exists\R$-complete ones will be found in the future.

\paragraph{Acknowledgments}
This work was supported by the Deutsche Forschungsgemeinschaft (DFG) grant 471183316 (ZA 1244/1-1).

\bibliographystyle{named}
\bibliography{main} 


\newpage
\appendix
\onecolumn

\section*{Appendix}

In this part, we present the missing proofs of the results that are included within the main paper. 

\section{Missing Proofs in Section~\ref{sec:sat:problems}}
\subsection{Proof of Theorem~\ref{thm:exist:equal:nexptime:hard}}

Theorem~\ref{thm:exist:equal:nexptime:hard} follows immediately from Theorem~\ref{thm:main:result} and Theorem~\ref{thm:NEXPTIME:sub:succETR:sub:EXPSPACE}.

Nevertheless, here we give an independent proof of Theorem~\ref{thm:exist:equal:nexptime:hard}, which does not depend on those theorems or the class $\csuccETR$.

\def\ccPP{\complexityclass{PP}}
\def\ccSHARPP{\complexityclass{\#P}}

\begin{proof}
We will reduce satisfiability of a Sch\"{o}nfinkel-Bernays sentence to satisfiability of $\SATprobpolysum$. 
The class of Sch\"{o}nfinkel--Bernays  sentences (also called Effectively Propositional Logic, EPR) is a fragment of first-order logic formulas where satisfiability is decidable. Each  sentence in the class is of the form $\exists \bx  \forall \by \psi$ whereby $\psi$ can contain logical operations $\wedge, \vee, \neg$, variables $\bx$ and $\by$, and relations $R_i(\bx,\by)$ which depend on a set of variables, but $\psi$ cannot contain any quantifier or functions. Determining whether a Sch\"{o}nfinkel-Bernays sentence is satisfiable is an $\ccNEXPTIME{}$-complete problem \cite{schoenfinkelLogicNEXPLewis1980} even if all variables are restricted to binary values \cite{schoenfinkelLogicBinaryNEXP2015}.  

\def\consistencyterm#1{f_C\llbracket\compactEquals{#1}\rrbracket}

We will represent Boolean values as the value of the random variables, with $ \pfalse=0 $ meaning \false{} and $ \ptrue=1 $ meaning \true{}.  
We will assume that $c=2$, so that all random variables are binary, i.e. $\mathit{Val} = \{\pfalse,\ptrue\}$. 
One could ensure such a setting by defining  $\PP{X=\ptrue} + \PP{X=\pfalse} = 1$ for each variable $X$ in the model.

We use random variables  
$ \bX=\{X_1,\myldots X_n\}$ and $\bY= \{Y_1,\myldots Y_m\}$ for the  quantified Boolean variables $\bx,\by$
in the sentence $\exists \bx  \forall \by \psi$ and define $n+m$ constraints: 
 $\PP{V=\ptrue} = \sfrac{1}{2}$ for every variable $V\in \bX \cup \bY$. (We use fractions in this proof to improve its readability; they can be removed  by multiplying the entire equation with a sufficiently large integer.)

 Besides  $ \bX$ and $\bY $, there will be identity testing variables~$ \bI $, relations variables~$ \bR $, logical operation variables~$ \bL $, and one final variable $ C $.
We denote all the variables as $\bV$. 
Variable $ C $ is used to verify that the values of all other variables are consistent, i.e., the probability $ \PP{C=\ptrue\mid \bx,\by,\bi,\br,\bl} $ should be $ \sfrac{2}{3} $ if the values are consistent and $ \sfrac{1}{3} $ otherwise. 
As an abbreviation for consistency that, e.g.,~$A=a,B=b,\myldots$,
we will write $\consistencyterm{A=a,B=b,\myldots}$ for 
$$
  \begin{array}{l}
     \sum_\bo \left( \sfrac{1}{3}  -  \PP{C=\ptrue \mid \bo,A=a,B=b,\myldots } \right)^2  
  \end{array}
 $$ 
where $A, B,\myldots$ can be any variables and $\bO=\bV\setminus \{C,A,B,\myldots\}$ are all other variables. 
The purpose of $\consistencyterm{\myldots}$ is to mark some variable values as \emph{inconsistent}.
Moreover, we will use the  abbreviation $f_C(c)$ for 
$$ 
\sum_{\bx,\by,\bi,\br,\bl} \big( \PP{c \mid \bx,\by,\bi,\br,\bl} - \sfrac{1}{3} \big)^2 
   \big( \PP{c \mid \bx,\by,\bi,\br,\bl}- \sfrac{2}{3} \big)^2 .
$$

Identity variable $ I_{i,j} \in \bI$ tests the equality of two variables $ V_i$ and $V_j $ in $\bX\cup \bY$. 
The value of $ I_{i,j} $ is consistent if $ v_i = v_j $ and $ i_{i,j } = \ptrue $, or 
$ v_i \neq v_j $ and $ i_{i,j} = \pfalse $. 
Let $f_{I_{i,j}} $ denote
\begin{align*}
& \consistencyterm{I_{i,j} = \pfalse, V_i = V_j = \ptrue}
+ \consistencyterm{I_{i,j} = \pfalse, V_i = V_j = \pfalse }\  + 
\consistencyterm{I_{i,j} = \ptrue, V_i = \pfalse, V_j = \ptrue }
+ \consistencyterm{I_{i,j} = \ptrue, V_i = \ptrue, V_j = \pfalse }.
 \end{align*}



For an $l$-ary relation $R_i(V_1,\myldots,V_l)$, we define a variable $R_i$.
If $\PP{r_i \mid v_1,\myldots,v_l} = \sfrac{2}{3}$, then it is understood as the value of the relation $R_i(V_1,\myldots,V_l)$ being $r_i$. 
$\PP{r_i \mid v_1,\myldots,v_l}$ is not constrained further than being in $ \{\sfrac{1}{3}, \sfrac{2}{3}\} $, which allows the relation to be chosen freely.
The probability table of variable $ C $ is tied to the probability of $ R_i $, such that it encodes inconsistency for values $r_i, v_1,\myldots,v_l$ when $\PP{r_i \mid v_1,\myldots,v_l}$   is not $ \sfrac{2}{3} $.
That means intuitively, that the existentially chosen probability distribution encodes for all $v_1,\myldots,v_l$ a chosen value $r_i$ and the consistency ensures that the value of $R_i$ is the chosen $r_i$.
Denote by $f_{R_i}(r_i) $ the formula
\begin{align*}
& \sum_\bv \left( \PP{r_i | v_1,\myldots,v_l} - \sfrac{2}{3} \right)^2 \left(  \PP{r_i | v_1,\myldots,v_l} - \sfrac{1}{3} \right)^2 +
\sum_{\bv} \left( \PP{r_i | v_1,\myldots,v_l} - \sfrac{2}{3} \right)^2 \left(  \PP{C=\ptrue | r_i, v_1,\myldots, v_l} - \sfrac{1}{3} \right)^2 .
\end{align*}



However, this construction only works, if this relation  $ R_i $ is always only applied to the same variables in the same order in the sentence $ \psi $. Hence, if the same relation occurs with different parameters 
as, e.g.,  $ R_i(V'_1,\myldots,V'_l) $, we consider this relation as a different relation $ R'_i $ 
and add a consistency constraint that $ v_1=v'_1, \myldots, v_l=v'_l $ implies $ r_i = r'_i $.
Let $ \bI_i= \{I_1,\myldots,I'_l \}$ be the identity testing variables 
for $ V_1, V'_1,\myldots,V_l,V'_l $. 
Then the subformula $ f_{R'_i} $ is defined as $f_{R_i}$ above 
with two additional terms 
$\consistencyterm{\bI_i = {\bf \ptrue}, R_i = \pfalse, R'_i = \ptrue}  
+ \consistencyterm{\bI_i = {\bf \ptrue}, R_i = \ptrue, R'_i = \pfalse}$.

If the inputs $V_1,\myldots,V_l$ and $V'_1,\myldots,V'_l$ to the relation contain some of the same variables in a different order, the consistency constraint would not cover the entire range of inputs. Nevertheless, we can assume that the same variables always occur in the same order by introducing additional variables in the Boolean formula before constructing the probabilistic formula, e.g., replacing $R(a, b) \wedge R(b, a)$ with 
$\forall s,t:  (s=b\wedge t=a) \Rightarrow R(a,b)\wedge R(s,t)$ where $\forall s,t$ and the implication are added immediately after the prefix~$\forall\by$.


Now that we have encoded identities and relations, we need to evaluate the logical operations of the sentence $ \psi $ on these atoms.  The sentence $\psi$ can be considered as a tree with relations/identities as leaves and logical operators as inner nodes. 
Leaf nodes correspond to the relation/identity nodes described above, and for the $i$th inner tree node, we add a logical node labeled by a variable $L_i$. The values of these variables are required to be consistent with the results of the operations in the tree.

Let $ L_i $ be a variable corresponding to a logical operation. If this operation is $ \neg L_j $ (with $ L_j $ corresponding to a variable $ L_j \in \bI\cup\bR\cup\bL $), the constraint $ l_i \neq l_j $ is encoded as
$$
\begin{array}{rcl}
f_{L_i} &\termdef& \consistencyterm{L_i = \ptrue, L_j = \ptrue}
+ \consistencyterm{L_i = \pfalse, L_j = \pfalse } .
\end{array}
$$
If the operation is $ L_j \wedge L_k $,  it is encoded as $f_{L_i} \termdef$
\begin{align*} 
&\consistencyterm{L_i = \pfalse, L_j = \ptrue, L_k = \ptrue } 
+ \consistencyterm{L_i = \ptrue, L_j = \pfalse, L_k = \ptrue } 
+\\[1mm]
&\consistencyterm{L_i = \ptrue, L_j = \ptrue, L_k = \pfalse } 
+ \consistencyterm{L_i = \ptrue, L_j = \ptrue, L_k = \pfalse  } .
\end{align*}

%
%
Other logical operations can be encoded analogously 
or be replaced by operations $ \neg$ and $\wedge $  in a preprocessing step.

Let $ L_q $ be the variable of the last logical node, i.e., 
the root of the tree and the final truth value of $ \psi $
and 
let 
$f_{\bi\br\bl} (x_1,\myldots,x_n, y_1,\myldots,y_m) \termdef$
\begin{align*}
& \sum_{\bi,\br,\bl\setminus l_q} 
\left( 3 \PP{C=\ptrue | x_1,\myldots,x_n, y_1,\myldots,y_m,\bi,\br,\bl\setminus l_q, L_q = \ptrue } -1 \right)^2, 
\end{align*}
and the full term as sum of all constraints $f(x_1,\myldots,x_n) \termdef$
\begin{align} \label{eq:full:term}
&   \sum_c f_C(c) + \sum_{i,j} f_{I_{i,j}} + \sum_{i,r_i} f_{R_i}(r_i) +   \nonumber \\ 
  &
   \sum_i f_{L_i} + (2^{m} - \sum_{\by}f_{\bi\br\bl}(x_1,\myldots,x_{n}, \by) )^2 .
\end{align}

We claim that the given Sch\"{o}nfinkel--Bernays sentence  $\exists \bx  \forall \by \psi$ 
is \true{} if and only if, for the expression $f$, 
there  exists a model $\fM$ and a sequence of values
$x_1,\myldots,x_{n}$ such that 
\begin{equation}\label{eq:model:final}
 f(x_1,\myldots,x_{n}) = 0   \ \wedge \bigwedge_{ V\in \bX \cup \bY}  \PP{V=\ptrue} = \sfrac{1}{2} .
\end{equation}

Now we add random variables to represent this sequence of values in the model: For $\bX=\{ X_1,\myldots, X_n \}$, define new random variables $\hat{X}_1,\myldots, \hat{X}_n$.
We claim that the given Sch\"{o}nfinkel--Bernays sentence  $\exists \bx  \forall \by \psi$ 
is \true{} if and only if   
there  exists a model $\fM$ such that 
\begin{align} \label{eq:model:final}
& \sum_{x_1,\ldots, x_n} \left(\PP{\hat{X_1}=x_1,\myldots, \hat{X_n}=x_n } \right)^2 =1  \ \wedge\nonumber \\[0mm]
& \sum_{x_1,\ldots, x_n}  \left(\PP{\hat{X_1}=x_1,\myldots, \hat{X_n}=x_n } \cdot f(x_1,\myldots,x_n) \right)= 0   \ \wedge \nonumber \\[0mm]
 & \bigwedge_{ V\in \bX \cup \bY}  \PP{V=\ptrue} = \sfrac{1}{2} .
\end{align}
is satisfied by   $\fM$.
Before giving a  proof that the claim is true, let us consider an example to illustrate our construction.
\begin{example}\label{ex:Sch-B}
Consider the following Sch\"{o}nfinkel--Bernays sentence:  
$
    \exists x \forall y\  \neg R(x,y).
$
It is \true{} since, e.g., the relation $x<y$ is a model which satisfies the sentence.
According to our construction, in the reduction, we use the following random variables: $X,Y,R,L,$ and $C$.
In this case, no $\bI$ variables are needed.
A model $\fM=(\cF,P,\bU,\bV)$ which satisfies the formula~\eqref{eq:full:term}, consists of five 
functions $x=u_X$, $y=u_Y$, $r=F_R(x,y,u_R)$, $l=F_L(x,y,r,u_L)$, and $C=F_C(x,y,r,l,u_C)$
and the probability distribution $P$ over $\bU$ such that it gives 
$P(X=\ptrue)=P(Y=\ptrue)=1/2$, and
$$
P(R=r\mid x,y)=\begin{cases}
2/3 & \text{if $r = \ptrue$ and $x<y$} \\
2/3 & \text{if $r = \pfalse$ and $x \ge y$} \\
1/3 & \text{otherwise}
\end{cases}
$$
(we define that $\pfalse < \ptrue$).

Next, if $P(R=r \mid x,y)=2/3$, then  $P(L= \neg r \mid x,y,r)=2/3$ and otherwise 
$P(L= l \mid x,y,r)=1/3$; and finally the conditional probabilities $P(C= c \mid x,y,r,l)$ are shown 
in Fig.~\ref{fig:example:dist:exaple}. It is easy to see that for this model and for $x=\ptrue$, 
the relation~\eqref{eq:model:final} is true.

\begin{figure*}[h]
$$
\begin{array}{cccc|cc}
X& Y & R &L&C=\pfalse& C=\ptrue\\ \hline
\rowcolor{lightgray}
\pfalse& \pfalse &  \pfalse & \ptrue & \sfrac{1}{3} & \sfrac{2}{3}\\
\pfalse& \ptrue &  \pfalse & \ptrue & \sfrac{2}{3} & \sfrac{1}{3}\\
\rowcolor{lightgray}
\ptrue& \pfalse & \pfalse & \ptrue & \sfrac{1}{3} & \sfrac{2}{3}\\
\rowcolor{lightgray}
\ptrue& \ptrue & \pfalse & \ptrue& \sfrac{1}{3} & \sfrac{2}{3}\\ \hline
\pfalse& \pfalse & \ptrue & \pfalse & \sfrac{2}{3} & \sfrac{1}{3}\\
\rowcolor{lightgray}
\pfalse& \ptrue & \ptrue & \pfalse & \sfrac{1}{3} & \sfrac{2}{3}\\
\ptrue& \pfalse& \ptrue & \pfalse & \sfrac{2}{3} & \sfrac{1}{3}\\
\ptrue& \ptrue & \ptrue & \pfalse & \sfrac{2}{3} & \sfrac{1}{3}
\end{array}
\hspace*{6mm}
\begin{array}{cccc|cc}
X& Y & R &L&C=\pfalse& C=\ptrue\\ \hline
\pfalse& \pfalse &  \pfalse & \pfalse & \sfrac{2}{3} & \sfrac{1}{3}\\
\pfalse& \ptrue &  \pfalse & \pfalse & \sfrac{2}{3} & \sfrac{1}{3}\\
\ptrue& \pfalse & \pfalse & \pfalse & \sfrac{2}{3} & \sfrac{1}{3}\\
\ptrue& \ptrue & \pfalse & \pfalse& \sfrac{2}{3} & \sfrac{1}{3}\\ \hline
\pfalse& \pfalse & \ptrue & \ptrue& \sfrac{2}{3} & \sfrac{1}{3}\\
\pfalse& \ptrue & \ptrue & \ptrue & \sfrac{2}{3} & \sfrac{1}{3}\\
\ptrue& \pfalse& \ptrue & \ptrue & \sfrac{2}{3} & \sfrac{1}{3}\\
\ptrue& \ptrue & \ptrue & \ptrue& \sfrac{2}{3} & \sfrac{1}{3}
\end{array}
$$
\caption{Conditional probabilities $P(C= c \mid x,y,r,l)$  satisfying the constraint \eqref{eq:full:term} 
for the reduction from the sentence 
$
    \exists x \forall y\  \neg R(x,y)
$
in Example~\ref{ex:Sch-B}.
In light-gray, we mark rows with  $P(C= c \mid x,y,r,l)=2/3$, i.e.~rows where the 
values of variables $X,Y,R,$ and $L$ consistently encode the relation $R(x,y)$, resp. $\neg R(x,y)$
for  $R(x,y)\equiv x<y$.
}
\label{fig:example:dist:exaple}
\end{figure*}

\end{example}

If $ f(\fx_1,\myldots,\fx_{|\bX|}) = 0 $ for some variable assignment and 
a model $\fM=(\cF,P,\bU,\bV)$, the probability distribution $P$ 
satisfies all consistency constraints and 
$\sum_{\by}  f_{\bi\br\bl}(\fx_1,\myldots,\fx_{|\bX|},\by) = 2^{|\bY|}$. 
For each variable in $ \bI,\bR,\bL $, half the possible assignments are inconsistent, and thus have $ \PP{C=\ptrue \mid   \bx, \by, \bi, \br, \bl } = \sfrac{1}{3} $. 
Hence, $ f_{\bi\br\bl} $ is at most 1. Since the formula encodes the logical operations, 
$ \left( 3 \PP{C=\ptrue \mid   \bx, \by, \bi, \br, \bl \setminus L_q, L_q =\ptrue} -1 \right)^2 $ is only 1, if the evaluation is consistent and the sentence $ \psi $ is \true{}. The sum $\sum_{\by} f_{\bi\br\bl}(\fx_1,\myldots,\fx_{|\bX|},\by)$ is $ 2^{|\bY|} $ if and only if $ \exists \bx \forall \by \psi $ is \true{}.

For the other direction of the reduction, we need to show: if $ \exists \bx \forall \by \psi $ is \true{}, then 
$ f(\fx_1,\myldots,\fx_{|\bX|}) = 0 $ for some variable assignment and 
a model $\fM=(\cF,P,\bU,\bV)$.
Let $ \bx $ be the Boolean values chosen by the existential operator  and $ P $ such that 
$ P(R_i =\ptrue\mid  v_1,\myldots,v_l) = \sfrac{2}{3}$ if relation $ R_i(v_j) $ 
is \true{} and $P(C=\ptrue \mid   \bx, \by, \bi, \br, \bl ) =\sfrac{2}{3} $ 
whenever possible without violating a constraint. Such a $ P $ exists because there are no contradicting constraints, each constraint only forces	 $ P $ to be $ \sfrac{1}{3} $ for some values.
Since $ \psi $ evaluates to \true{} for all $ \by $, there is an evaluation of the logical operators in $ \psi $ that returns \true{}, so setting $ \bI,\bR,\bL $ to the values that correspond to such an evaluation results in $ f_{\bi\br\bl} $ being $ 1 $ for all $ \by $. Then the sum $\sum_{\by}f_{\bi\br\bl}(\bx,\by)$ is $ 2^{|\bY|} $, and thus $ f (\bx) = 0 $.
\end{proof}

The proof shows that the $\ccNEXPTIME$-hardness occurs because the probability distribution can store a large amount of data and the formulas can read all of it. Even though our constructions use only very simple probabilities for primitives, mostly $1/3$ or   $2/3$, and simple formulas 
the $\SATprobpolysum$ problem is $\ccNEXPTIME$-hard. The study of the class $\csuccETR$ shows the increased hardness resulting from probabilities that are real numbers.

\newcommand{\VALprobpolysum}{\mbox{\sc Val}_{\textit{prob}}}
\newcommand{\coVALprobpolysum}{\overline{\mbox{\sc Val}}_{\textit{prob}}}
\newcommand{\coSATprobpolysum}{\overline{\mbox{\sc Sat}}_{\textit{prob}}}
\newcommand{\VALcausalpolysum}{\mbox{\sc Val}_{\textit{\subscript}}}

\subsection{Proof of Corollary~\ref{corr:main}}
\begin{proof}
We show the ${\tt co}\mbox{-}\csuccETR$-completeness of the validity problem for the 
language $\Lprobpolysum$. The proof for $\Lcausalpolysum$ is analogous and we skip it.

Let us denote the language containing the \emph{yes} instances of the validity problem as 
$$
   	\VALprobpolysum = \{\varphi \in \Lprobpolysum : \   \mbox{for all $\fM$  we have $\fM \models \varphi$} \}.
$$ 
Obviously, it is true, that for all $\varphi \in \Lprobpolysum$
$$
	\varphi \in \VALprobpolysum \quad  \Leftrightarrow \quad \neg \varphi \in  \coSATprobpolysum .
$$

We prove first the ${\tt co}\mbox{-}\csuccETR$-hardness of $\VALprobpolysum$.
Let $A$ be any language in ${\tt co}\mbox{-}\csuccETR$. By definition, we have 
$\overline{A}\in \csuccETR$ and,  from Theorem~\ref{thm:main:result}, we know that 
$\overline{A} \lep \SATprobpolysum$. Let $f$ denote the function which transforms
(in polynomial time) an instance $x$ of  $\overline{A}$ to an $\Lprobpolysum$-formula 
$f(x)=\varphi$ such that 
$$
	x\in \overline{A}\quad  \Leftrightarrow \quad \varphi\in  \SATprobpolysum
$$
or equivalently that 
$$
	x \in {A} \quad  \Leftrightarrow \quad
	\varphi \in  \coSATprobpolysum \quad \Leftrightarrow \quad
	\neg\varphi\in  \VALprobpolysum.
$$
Thus,  for the reduction $A \lep \VALprobpolysum$, we can use the transformation 
$x \mapsto \neg f(x)= \neg \varphi$.

To prove  that $\VALprobpolysum\in {\tt co}\mbox{-}\csuccETR$, we need to show that 
the complement  $\coVALprobpolysum$ is in $\csuccETR$.  To this end, we use 
$\SATprobpolysum$ which,  due to Theorem~\ref{thm:main:result}, is complete for 
$\csuccETR$ and conclude that $\coVALprobpolysum\lep \SATprobpolysum$ via 
a straightforward transformation: $ \varphi \mapsto \neg\varphi$. Obviously, we have
$$
  	\varphi\in \coVALprobpolysum \quad  \Leftrightarrow \quad \neg\varphi\in \SATprobpolysum
$$
which completes the proof that $\VALprobpolysum $ is ${\tt co}\mbox{-}\csuccETR$-complete.

Finally, to show the inclusions 
$$
	{\tt co}\mbox{-}\ccNEXPTIME \subseteq {\tt co}\mbox{-}\csuccETR\subseteq \ccEXPSPACE,
$$
assume first $A\in {\tt co}\mbox{-}\ccNEXPTIME$. Then, by definition, $\overline{A}\in \ccNEXPTIME$ and, due to Theorem~\ref{thm:NEXPTIME:sub:succETR:sub:EXPSPACE}, we get $\overline{A}\in \csuccETR$
which means $A\in {\tt co}\mbox{-}\csuccETR$. If $B\in {\tt co}\mbox{-}\csuccETR$, then by definition 
$\overline{B}\in \csuccETR\subseteq \ccEXPSPACE$, and since the last class is closed under the complement,
we conclude $B\in \ccEXPSPACE$.
\end{proof}

\section{Missing Proofs in Section~\ref{sec:succETR}}
\subsection{Proof of Theorem~\ref{thm:NEXPTIME:sub:succETR:sub:EXPSPACE}}

\begin{proof}
The succinct version of the satisfiability problem is
known to be $\ccNEXPTIME$-complete. A (Boolean) satisfiability instance
can be easily embedded into an $\ETR$ instance. We only need to constrain the real
variables to be $\{0,1\}$-valued. This is achieved by adding the equations $x_i(x_i-1) = 0$.
This embedding is local. Therefore, given the circuit of a succinct satisfiability instance,
we can transform it into a circuit of an equivalent ETR instance.

The upper bound follows from the inclusion $\cETR \subseteq \ccPSPACE$.
Given the circuit $C$, we can construct the formula explicitly,
which has an exponential size in the size of $C$. Running the $\ccPSPACE$
algorithm on this instance yields an $\ccEXPSPACE$ algorithm.
\end{proof}

\section{Missing Proofs in Section~\ref{sec:first:compl:problems}}

\subsection{Proof of Lemma~\ref{lem:succQUAD-hard}}

\begin{proof}
\citeauthor{DBLP:journals/mst/SchaeferS17}~[\citeyear{DBLP:journals/mst/SchaeferS17}]
give a similar reduction for $\ETR \lep \QUAD$ based on a trick by Tseitin.
Our reduction follows their construction, but we have to deal with the succinctness.

Let $C$ be an instance of $\succETR$, assume that the variables are $x_i$
with $i \in \{0,1\}^K$ and the nodes of the formulas are given by bit strings of
length $N$. The quadratic polynomials are polynomials in the variables $x_i$ 
and additional variables $y_j$, $j \in \{0,1\}^{N+O(1)}$.
(Technically, the variables in the $\succQUAD$ instance are all given by bit strings
and the first bit will decide whether it is an $x$- or a $y$-variable.)

Our reduction will have the following property:
Whenever $\xi \in \Rset^{\{0,1\}^K}$ satisfies the formula given by $C$,
then there is an $\eta \in \Rset^{\{0,1\}^{N + O(1)}}$  
such that $(\xi,\eta)$ is a root of all polynomials of the $\succQUAD$ instance.
And on the other hand, whenever 
$(\xi,\eta)$ is a root of all polynomials of the $\succQUAD$ instance,
then $\xi$ will be a satisfying assignment of the formula given by $C$.
Our reduction will even ensure that each quadratic polynomial
has at most four monomials.

Given the circuit $C$, we construct a circuit $D$ that encodes
a list of quadratic polynomials.
We will construct one or two 
quadratic polynomials for each node $v$
of the encoded formula. Since we need to construct a circuit $D$,
we want the structure to be as regular as possible.
So we will always add two polynomials, the second one being potentially
the zero polynomial.

For each node, we introduce a constant number
of new variables $y_v$, $y_v'$, and maybe $y_v''$. The latter two variables
are only for intermediate calculations. (Formally, we index these three variables
with strings $v00$, $v01$, and $v10$, for instance.) 
Let $v_1$ and $v_2$ be the children of $v$ and $v_0$ its parent.

The reduction will have the following property: In any assignment $(\xi,\eta)$
that is a root of all quadratic polynomials, the value assigned to $y_v$
by $\eta$ will correspond to the value of the node in the given formula
when we assign the values $\xi$ to the $x$-variables.
For the Boolean part of the formula, this means that if the value assigned to $y_v$
is zero, then the subformula rooted at $v$ is true. The reverse statement needs not to be
true, however, since the Boolean part of the formula is monotone, this
does not matter. (This is the part in which we need that the negations are pushed down,
cf.\ Remark \ref{remark:pushdown}).

The input of the circuit $D$ is of the form $(vb,m)$:
$v$ is a bit string of length $K$ and $b$ is a single bit. 
This means that we want the first or second polynomial of the two polynomials
belonging to node $v$. The bit string $m$ addresses the monomial
of the polynomial given by $vb$.
Since all our polynomials have at most four monomials,
$m$ will only have two bits. 

The circuit $D$ 
first checks whether the labeling of the parent $v_0$ is consistent with the labeling
of $v$, that is, if $v$ is labeled with an arithmetic operation,
then $v_0$ is labeled with an arithmetic operation or comparison.
If $v$ is labeled with a comparison, then $v_0$ is labeled with a Boolean operation.
And if $v$ is labeled with a Boolean operation, so is $v_0$.
If this is not the case, then the quadratic polynomial associated with 
$v$ will be $1$. (In this case, $C$ does not encode a valid $\ETR$ instance. 
A non-valid encoding is a no-instance, so $D$ should encode
a no-instance.) 

If a node $v$ is labeled with a constant $c$, then the polynomial is $y_v - c$.
(In $\succETR$, we only start with the constants $0$ and $1$, so $c$ is either $0$
or $1$. The proof would however also work, if we allow more constants in $\succETR$
as syntactic sugar.)
If $v$ is labeled with a variable $x_i$, then the polynomial is $y_v - x_i$.

If the node $v$ is labeled with an arithmetic operation $\circ$, then the 
polynomial will be $y_v - y_{v_1} \circ y_{v_2}$. 

If the label of $v$ is $=$, then the polynomial is $y_v - (y_{v_1} - y_{v_2})$.
If the label of $v$ is $<$, then we have two polynomials,
$y_v' - (y_v'')^2$ and $(y_{v_2} - y_{v_1}) y_v' - (1 - y_v)$. 
When both polynomials are zero and $y_v$ is assigned zero, then
since $y_v' \ge 0$, $y_{v_2} > y_{v_1}$. (The reverse might not be true
and cannot be achieved in general with polynomials, since
the characteristic function of $\ge$ has an infinite number of zeroes.)

If $v$ is labeled with $\vee$,
then the polynomial is $y_v - y_{v_1} y_{v_2}$. If $v$ is labeled
with $\wedge$, then the polynomial is $y_v - (y_{v_1}^2 + y_{v_2}^2)$.

Finally, if $v$ is the root, then we also add the polynomial $y_v$.
This will ensure that any common root of the quadratic polynomials
yields a satisfying  assignment.

If we have a satisfying assignment of the $\ETR$ instance, then it is obvious
that we can extend it to a common root $(\xi,\eta)$ of the quadratic system
such that $\eta$ corresponds to the values of intermediate nodes in the formula.

On the other hand, if we have a common root $(\xi,\eta)$ of the quadratic system,
then $\xi$ is a satisfying assignment. The value of $\eta$ corresponding to a $<$-node
could be nonzero although the Boolean value in the formula is true.   
However, since the Boolean part of the formula is monotone, the result
will be true if we flip the value to the right one.

Given $C$, we can construct the circuit $D$ in polynomial time (using $C$
as a subroutine).
\end{proof}

\subsection{Proof of Lemma~\ref{lemma:seq:of:reductions} and Theorem~\ref{thm:three:complete:problems}}

In this section, we show the reductions:
	$$
	\succQUAD \lep \sigmaETR 
	\lep \succETR .
	$$
To this purpose, we define two new problems: $\succFEASfour$ and $\sigmapiETR$
and show in Lemmas~\ref{lem:succ4-feas-hard} to \ref{lem:sigmapiETR}
below, the problems are reducible to each other in a ring:
	$$
	\succQUAD  \lep 
	\succFEASfour \lep 
	\sigmaETR \lep 
	\sigmapiETR \lep 
	\succETR
	$$
which, together with Lemma~\ref{lem:succQUAD-hard} showing that 
$\succETR \lep \succQUAD$, completes the proof of  Lemma~\ref{lemma:seq:of:reductions} and  Theorem~\ref{thm:three:complete:problems}.

We start with $\FEASfour$ which is the following problem: Given a multivariate polynomial $p$ 
of degree at most four with rational coefficients, has it a root? 
$\FEASfour$ is a complete problem for the class $\cETR$ \cite{DBLP:journals/mst/SchaeferS17}.
The succinct version $\succFEASfour$ 
is defined as follows: For $\succFEASfour$, we get
a Boolean circuit $C$ as input computing a function 
$\{0,1\}^N \to \{0,1\}^M$. On input $x$ (which we interpret as a number in binary),
$C(x)$ outputs a description of the $x$th monomial of $p$, that is, it outputs
the coefficient as a quotient of two integers given
in binary as well as the names of the four variables occurring in the monomial
(repetitions are allowed and if the degree is lower than four, then we output $1$
instead of a variable).

\begin{lemma}\label{lem:succ4-feas-hard}
$\succQUAD \lep \succFEASfour$.
\end{lemma}
\begin{proof}
The basic idea is very simple, a common root of polynomials $p_1,\dots,p_\ell$ over 
$\Rset$ is a root of $p_1^2 + \dots + p_\ell^2$ and vice versa.

Let $C$ be a circuit representing a $\succQUAD$ instance,
computing a function $\{0,1\}^K \times \{0,1\}^N \to \{0,1\}^M$.
The circuit $D$ representing the sum of squares computes
a function $\{0,1\}^{K + 2N} \to \{0,1\}^{M'}$. 
On input $(k,u,v)$, $D$ outputs the products of the monomials $C(k,u)$ and 
$C(k,v)$.
\end{proof}

\begin{lemma}\label{lem:succFEASfour:to:sigmaETR}
$\succFEASfour \lep \sigmaETR$.
\end{lemma}
\begin{proof}
We can think of having five circuits $C_0,\dots,C_4: \{0,1\}^M \to \{0,1\}^N$. 
$C_0(j)$ gives the coefficients of the $j$th monomial
and $C_1(j),\dots,C_4(j)$ are the indices of the four variables (repetitions
are allowed and the variables can be~$1$ if the degree of the monomial
is smaller). Let $\hat C_0,\dots \hat C_4$ be the arithmetizations of 
$C_0,\dots,C_4$, that is, we consider the Boolean variables
as variables over $\Rset$ and replace $\neg x$ by $1 - x$ and $x \wedge y$
by $x \cdot y$. 
Let $o_{i,k}(j)$ be the $k$th output bit of $\hat C_i$ on input $j$.  
The expression
\[
\sum_{e_1=0}^1\ldots\sum_{e_N=0}^1 
x_{n(e_1,\dots,e_N)}
   \prod_{k = 1}^N (o_{i,k}(j) e_k + (1 - o_{i,k}(j))(1 - e_k))
\]
acts as a ``variable selector'', it is exactly the variable given by $C_i(j)$, with $i=1,\ldots,4$.
Note that the inner product is not an exponential one.
In the same way, we can build a ``constant selector'' of $C_0$ by
replacing $x_{n(e_1,\dots,e_N)}$ by $\sum_{s = 1}^N 2^s e_s$. 
(It is enough to consider integer constants). 
Now we get an expression for the $\succFEASfour$ by summing over all $j$
and multiplying the constant selector with the four variable selectors.

The output is almost an $\sigmaETR$ instance, except that we use arithmetic
circuits for the computation of the $o_{i,k}(j)$. 
They can be replaced by formulas using
Tseitin's trick and introducing a new variable for every gate.
\end{proof}

$\sigmapiETR$ is defined in a similar way as $\sigmaETR$. 
Besides exponential sums, we allow exponential products, too.

\begin{lemma}\label{lem:sigmaETR:tosigmapiETR}
	$\sigmaETR \lep \sigmapiETR$.
\end{lemma}

The proof of this lemma is obvious.

\begin{lemma} \label{lem:sigmapiETR}
	$\sigmapiETR \lep \succETR$.
\end{lemma}

\begin{proof}
The idea is again to use Tseitin's trick: 
If we have, for instance, an exponential product
$\prod_{e_1=0}^1\dots\prod_{e_N=0}^1  F(e_1,\dots,e_N)$,
we create new variables $y_z$, $z \in \{0,1\}^{\le N}$.
($F$ might depend on more, free variables, but we suppress them here
for a simpler presentation.)
The product can be replaced by the equations
$y_z = F(z_1,\dots,z_n)$ for $z \in \{0,1\}^N$
and $y_z = y_{z0} y_{z1}$ for $z \in \{0,1\}^{< N}$.
In the same way, we get an exponential sum by 
taking the equations $y_z = y_{z0} + y_{z1}$ for $z \in \{0,1\}^{< N}$
instead. Since these equations are very regular, we can easily design
a circuit representing them.
\end{proof}

\section{Missing Proofs in Section~\ref{sec:completeness}}

We start with the following:
\begin{remark} \label{remark:extended-sigma-ETR}
In the definition of $\sigmaETR$ and $\sigmapiETR$, we allow 
variable indexing of the form $x_{n(x_{j_1},\dots,x_{j_m})}$
where $n(x_{j_1},\dots,x_{j_m})$ is the number given by the binary input string.
We can think of more complicated indexing functions,
namely $x_{g(x_{j_1},\dots,x_{j_m})}$ where $g$ is any Boolean function computed
by a circuit $D$ of polynomial size, which is part of the input. 
We can have several such index functions and corresponding 
circuits, more precisely, polynomially many.
The proof of the upper bound in Lemma~\ref{lem:sigmapiETR}
still remains valid, since at the leaf $y_z$ in the proof, we can evaluate the
circuit for $g$ to compute the index of the variables occurring in $F(z_1,\dots,z_n)$.
\end{remark}

This allows us to index variables in a more convenient and clear way
as we will use in the proof below.

\subsection{Proof of Proposition~\ref{prop:SAT:causal:in:sigmaETR}}

\begin{proof}
To prove that $\SATcausalpolysum \lep \sigmaETR$, 
suppose $\varphi \in \Lcausalpolysum$. 
We construct an instance $\psi$ of $\sigmaETR$
such that $\varphi$ is satisfied if and only if 
the sentence  $\psi$ is true.

Let $B$ be a Boolean formula over $\wedge, \vee, \neg$ and arithmetic terms 
$t_1,\ldots,t_k$ of the form 
$s_i < s'_i, s_i \le s'_i, s_i = s'_i$, where $s_i, s'_i$
are polynomials over probabilities 
such that 
$\varphi=B(t_1,\ldots,t_k)$.  W.l.o.g., we will assume that there is 
no negation in $B$:
one can note that eliminating the negations can be done by pushing them to the bottom
 of the formula via De Morgan’s laws and replacing $\neg(t\le t')$ by $t > t'$, etc.

Denote, by $\Ext(s)$, the extension of a polynomial $s$ in which all sub-expressions 
of the form $\sum_{\ell} e(\ell)$ are replaced by the sum 
$e(0)+e(1)+\ldots +e(\maxvaluecount - 1)$ and let  
$\Ext_{\varphi} = \bigcup_{i=1}^k (\Ext(s_i)\cup \Ext(s'_i))$.
Let $\Phi$ be the set of  $\Lfull$-formulas appearing inside $\Ext_{\varphi}$.
For example, if $s_i=\sum_{z} \PP{X=1, Z=z}$ for some $i$, then 
$(X\mbox{=}1 \wedge Z\mbox{=}0)$ and $(X\mbox{=}1 \wedge Z\mbox{=}1)$ belong 
to $\Phi$ (assuming binary domain  $\mathit{Val} = \{0,1\}$).
Note that each expression in $\Ext_{\varphi}$ is a polynomial in probabilities 
$\{\PP{f}\}_{f\in \Phi}$.

Next, let $X_1,\ldots, X_n$ be the random variables 
appearing in $\Phi$. 
Define 
$$
	\Delta_{\textit{prop}}=\{ \beta_1\wedge \ldots \wedge \beta_n: 
	\text{$\beta_i \in \{X_i\mbox{=}0,  X_i\mbox{=}1,\ldots, X_i\mbox{=}c-1 \}$ \ for all\ $1\le i\le n$ } \}.
$$
Let $\alpha_1,\ldots,\alpha_l$ be the interventional  antecedents, i.e., an $\Lint$-formula
appearing in $[\alpha_i]$ in any formula in   $ \Phi$ and define
$$\Delta=\{[\alpha_i] \delta_{\textit{prop}}: 1\le i\le l,  \delta_{\textit{prop}}\in \Delta_{\textit{prop}}\}.$$
 
 
For $i=1,\ldots, k$, let $G_i$ denote the set of 
formulas appearing in probabilities inside the term 
$t_i$ of $\varphi$ in which dummy variables for summation are used. 
E.g., if $s_i=\sum_{\ell} \PP{X=1, Z=\ell}$, 
then the formula $(X\mbox{=}1 \wedge Z\mbox{=}\ell)$ belongs to $G_i$.
For each $g\in G_i$ over dummy variables 
$\ell_1,\ldots,\ell_{t_g}$, we consider $g(\ell_1,\ldots,\ell_{t_g})$
as a function $$g:\mathit{Val}^{t_g} \to \Phi$$ such that 
substituting $\ell_1,\ldots,\ell_{t_g}$ by the values 
$c_1,\ldots,c_{t_g}$  we get $g(c_1,\ldots,c_{t_g})$ in  $\Phi$.

Now we are ready to describe the sentence $\psi$ of 
$\sigmaETR$. We use the following variables, which we index by formulas:
$$
  x_f \quad \text{for all $f\in \Phi$} \quad \text{and}\quad
  x_{\delta} \quad  \text{for all $\delta\in\Delta$}.
$$
   %
For conditioning, we include additional variables
$$x_{f|h} \quad \text{for all $f,h\in \Phi$}. $$

Let $\gamma(f) $ denote $f$ for all $f \in \Phi$ and let $\gamma(g)$ denote $g(\ell_1,\ldots,\ell_{t_g})$ 
for terms $g$ depending on dummy variables $\ell_1,\ldots,\ell_{t_g}$.
For every polynomial $e$ in $\Tpolysum(\Lfull)$, 
we denote by $\Gamma(e)$ 
the expression $e$ in which we substitute 
every atomic term $\PP{f}$ by the variable $x_{\gamma(f)}$ and 
every conditioned term $\PP{f\mmid h}$ by the variable $x_{\gamma(f)\mmid\gamma(h)}$.

Then, for every  $\varphi$'s term $t_1,\ldots,t_k$ 
of the form $s_i < s'_i, s_i \le s'_i, s_i = s'_i$, 
we define the arithmetic terms $\Gamma(t_i)$ in $\psi$ as follows: 
\begin{align}\label{caus:eq:basic:red:caus}
    \Gamma(s_i)  <  \Gamma(s'_i),\  \Gamma(s_i) \le  \Gamma(s'_i),\ \text{and} \ \Gamma(s_i) =  \Gamma(s'_i),
\end{align}
respectively.

Next, we add a  formula for conditioning: 
\begin{align}
    &\mbox{$\sum_{f,g,h\in\Phi} A(f,g,h) (x_{f|g} x_{g} - x_{h})^2 = 0 $, 
    } 
     \label{eq:cond}
\end{align}
where $A(f,g,h)$ is an expression defined below, such that $A(f,g,h)=1$ if $h$ is equal to $f\wedge g$ with the same antecedents, and $0$ otherwise. Thereby, sums like  $\sum_{f,g,h\in\Phi}$ are an abbreviation for three sums with variables $f,g,h$ over all elements of $\Phi$ (see Remark~\ref{remark:extended-sigma-ETR} on how to encode it in $\sigmaETR$).

Finally, we add   to the expressions~\eqref{caus:eq:basic:red:caus}
and  Eq.~\eqref{eq:cond} 
 the following formulas:
\begin{align}
    &\mbox{$\sum_{f\in \Phi}\left( x_f - 
       \sum_{\delta\in \Delta} E(f,\delta) \cdot x_{\delta}^2\right)^2 = 0$} 
     \label{caus:eq:red:psi1:caus}\\[1mm]
    &\mbox{$\sum_{i=1}^{l}((\sum_{\delta_{\textit{prop}} \in \Delta_{\textit{prop}}} 
       F(\alpha_i, \delta_{\textit{prop}}) \cdot x_{[\alpha_i]\delta_{\textit{prop}}}^2)-1)^2 = 0$}, \label{caus:eq:red:psi2:caus}
\end{align} 
where $E(f,\delta)$ is an expression defined below, such that 
$E(f,\delta) = 1$ if $f$ and $\delta$ have the same interventional antecedent $\alpha_i$,
i.e.,  $f=[\alpha_i]f'$ and  $\delta=[\alpha_i]\delta'$ for some $\alpha_i$ and 
the implication $\delta' \to f'$
is a tautology; otherwise $E(f,\delta) = 0$.
Moreover,  $F(\alpha_i, \delta_{\textit{prop}})=1$ if and only if $\alpha_i$ 
is consistent with $\delta_{\textit{prop}}$, i.e., for all $X_j=a_j$ in $\alpha_i$, 
we have also $X_j=a_j$ in $\delta_{\textit{prop}}$.
Note that it is true that $x_{\delta}^2 \ge 0$ for all $\delta\in\Delta$,
the value of the monomial $x_{\delta}^2$ encodes the probability  $\PP{\delta}$
and the value of variable $x_f$ encodes  the probability 
$\PP{f}$. 

Let $C: \{0,1\}^M \to \{0,1\}$ be a Boolean circuit 
such that $C(f,\delta) = 1$ iff $f,\delta$ satisfies the condition above.
We select $M$ such that the formulas $f\in \Phi$ and $\delta \in \Delta$
can be encoded by binary strings of length $M$.
Let $\hat C$ be the arithmetizations of $C$, that is, 
similarly as in the proof of Lemma~\ref{lem:succFEASfour:to:sigmaETR}, we consider 
the Boolean variables as variables over $\Rset$ and replace 
$\neg x$ by $1 - x$ and $x \wedge y$ by $x \cdot y$.
Let ${\hat C}(f,\delta)$ be the output of $\hat C$ on input $f,\delta$. 
Using Tseitin's trick and introducing a new variable for every gate,
we get from $\hat C$ an expression $E$ such that $E(f,\delta) = {\hat C}(f,\delta)$.
We define $F(\alpha_i, \delta_{\textit{prop}})$ in an analogous way.
$A(f,g,h)$ can be defined in the same way, starting with a circuit and applying the transformations.

Although the indexing functions $g(c_1,\ldots,c_{t_g})$ are more complicated than the functions in the immediate definition in $\sigmaETR$, one can use more complicated functions without changing the complexity, as explained in Remark~\ref{remark:extended-sigma-ETR}.


We complete the description of $\psi$ as follows:
\[
\psi := B(\Gamma(t_1),\ldots,\Gamma(t_k)) 
\wedge \mbox{Eq.~\eqref{eq:cond}}
\wedge \mbox{Eq.~\eqref{caus:eq:red:psi1:caus}}
\wedge \mbox{Eq.~\eqref{caus:eq:red:psi2:caus}}.
\]

\begin{example}
We consider an instance  $\varphi$ consisting, for simplicity,  of  purely probabilistic expressions and only 
one term (assuming binary domain  $\mathit{Val} = \{0,1\}$).: 
\begin{equation}\label{ex:proof:term}
  \mbox{$2\PP{Y=1}\ +\ $} \mbox{$ \sum_{\ell_1} \PP{X=1, Z=\ell_1} \  $} 
      \mbox{$\sum_{\ell_2 } \PP{Y=1, X=\ell_2, Z=\ell_1}$}   \  \le \  \mbox{$\sfrac{1}{2}$.}
\end{equation}
We have $\Phi =$
$$
  	\{Y\mbox{=}0, 
	(X\mbox{=}1 \wedge Z\mbox{=}0), 
	(X\mbox{=}1 \wedge Z\mbox{=}1),
	(Y\mbox{=}1 \wedge X\mbox{=}0 \wedge Z\mbox{=}0),
	(Y\mbox{=}1 \wedge X\mbox{=}0 \wedge Z\mbox{=}1),
	(Y\mbox{=}1 \wedge X\mbox{=}1 \wedge Z\mbox{=}0),
	(Y\mbox{=}1 \wedge X\mbox{=}1 \wedge Z\mbox{=}1)
	\}
$$
 and 
 $$
 \Delta=\Delta_{\textit{prop}}= \{ Y\mbox{=}y \wedge X\mbox{=}x \wedge Z\mbox{=}z: y,x,z\in  \{0,1\} \}.
 $$
Then, term~\eqref{ex:proof:term} is transformed to 
\begin{align*} 
  & 2x_{Y\mbox{=}1} + \sum_{\ell_1 = 0}^1 x_{X\mbox{=}1 \wedge Z\mbox{=}\ell_1} 
      \sum_{\ell_2 = 0}^1 x_{Y\mbox{=}1 \wedge X\mbox{=}\ell_2 \wedge Z\mbox{=}\ell_1}
      \ \le\ \sfrac{1}{2}
\end{align*}
and additionally, to the resulting $\psi$,  
Eq.~\eqref{caus:eq:red:psi1:caus}, and 
Eq.~\eqref{caus:eq:red:psi2:caus} are added for $\Phi$ and $\Delta$ as above.
\end{example}
%
%
%
%
%
\end{proof}

\subsection{Proof of Lemma~\ref{lemma:interv:succ:complete}}

The goal of this section is to prove that  $\succETRcR{\sfrac{1}{8}}{-\sfrac{1}{8},{\sfrac{1}{8}}}$ is $\csuccETR$-complete. 
We show this  in the subsequent lemmas:
First we show that $\succETRc{1}$ is  $\csuccETR$-complete and then prove that 
$\succETRc{1}\lep  \succETRcR{\sfrac{1}{8}}{-\sfrac{1}{8},{\sfrac{1}{8}}}$.
We follow ideas by \citet{abrahamsen2018art}, but we have to pay attention
to the succinctness.

Note that $\succETRc{1}$
is rather a restricted version of $\succQUAD$ than of $\succETR$ (in our notation),
however, we use the name $\succETRc{1}$ as introduced (for the non-succinct case)
by \citet{abrahamsen2018art}. Note that in the hardness proof of 
$\succQUAD$ (Lemma~\ref{lem:succQUAD-hard}), all quadratic equations have at most
four monomials and only use the constants $1$ and $-1$ as coefficients.

\begin{lemma}\label{lem:succQUAD:to:succETRc1}
$\succQUAD \lep \succETRc{1}$. 
\end{lemma}
\begin{proof}
Let $C$ be an instance of $\succQUAD$, the circuit $C$ computes a function
$\{0, 1\}^K \times \{0, 1\}^2 \to \{0, 1\}^M$. $C(s, t)$ is an encoding
of the $t$th monomial in the $s$th quadratic equation. Note that we can
assume that there are at most four monomials, so we need bit strings 
of length two. For every~$s$, we will have a constant number of 
equations of the three types. We label them by $(s,b)$.

First of all, we will have three distinguished variables $v_{-1}$, $v_0$,
and $v_{1}$, which will have the values $-1$, $0$, $1$.
This is achieved by the equations $v_{1} = 1$,
$v_{1} + v_0 = v_1$, and $v_{1} + v_{-1} = v_0$.

Every quadratic equation consists of four monomials.
Every monomial is of the form $c \cdot x \cdot y$. 
Thereby,  $c$ is a coefficient,
which is in $\{1,-1\}$,
and $x$ and $y$ are variables, or, if the monomial has a lower degree,
are the constant $1$. We will have two extra variables $m_{s,t,0}$ and $m_{s,t,1}$
for every monomial. 
We add the two new equations $m_{s,t,1} = x y$ 
and $m_{s,t,0} = c m_{s,t,1}$. (Strictly speaking, $c$ is replaced by $v_{-1}$ or $v_1$,
and $x$ or $y$ is replaced by $v_1$ if it is $1$.)
Then $m_{s,t,0}$ contains the value of the monomial. 
Furthermore, we have three equation variables $e_{s,t,0},e_{s,t,1},e_{s,t,2}$.
The three equations  $e_{s,t,1} = m_{s,00,0} + m_{s,01,1}$,
$e_{s,t,2} = m_{s,10,0} + m_{s,11,1}$, and $e_{s,t,0} = e_{s,t,1} + e_{s,t,2}$
ensure that $e_{s,t,0}$ contains the value of the quadratic polynomial. 
With the equation $e_{s,t,0} = v_0 + v_0$, we ensure that the value is zero.

Our seven circuits $C_1,\dots,C_7$ now work as follows: The circuit $C_1$
always outputs (an encoding of) $v_1$ since it is the only variable that is
assigned the constant $1$. Every quadratic equation is represented
by four equations of the form $y + z = x$  and eight equations 
of the form $yz = x$ (two for each monomial). The input of the circuits $C_1,\dots,C_7$
is interpreted as a string $(s,b) \in \{0,1\}^K \times \{0,1\}^3$.
On input~$s$, the circuit $C_2$, for instance, uses the circuit $C$
to obtain the $s$th quadratic equation of the $\succQUAD$ instance.
If $b$ starts with a $0$, then $b$ represents a number between $0$ and $3$,
and $C_2$ puts out the first variable of the $b$th equation of the form 
$y + z = x$ representing the $s$th quadratic equation.
If $b$ starts with a $1$, then $C_2$ outputs the first variable
of the equation $v_{1} + v_0 = v_1$ or $v_{1} + v_{-1} = v_0$.
(In this way, we do not need to take extra care of these two extra equations.
They get repeated a lot, but this does not matter.)
The circuits $C_3$ and $C_4$ work the same, but they put out the second or
third variable. 
The circuits $C_5$, $C_6$, and $C_7$ work in a similar fashion, outputting eight equations for each equation of $\succQUAD$.
\end{proof}

Next, we will reduce $\succETRc{1}$ to the problem $\succETRcR{\sfrac{1}{8}}{-\sfrac{1}{8},{\sfrac{1}{8}}}$.
Let $C$ be an instance of $\succETR$ encoding a formula $\varphi$ in $n$ variables.
Let 
\[
  S_\varphi = \{ (\xi_1,\dots,\xi_n) \in \Rset^n \mid \text{$\varphi(\xi_1,\dots,\xi_n)$ is true} \}
\]
be the solution set. The bit complexity $L$ of $\varphi$ is the length of
$\varphi$ written down in binary. 
If $s$ is the size of $C$, then $L \le 2^s$.
\citet{DBLP:journals/mst/SchaeferS17} prove the following lemma.

\begin{lemma} \label{lem:ETR:sol:size}
Let $B$ be the set of points with a distance at most $2^{2^{8n \log L}}$ from the origin
(in the Euclidean norm). Then $S_\varphi$ contains a point from $B$.
\end{lemma}

This allows us to prove the next lemma.

\begin{lemma}\label{lem:succETRc:1:to:succETRc:1_8}
$\succETRc{1} \lep \succETRcR{\sfrac{1}{8}}{-\sfrac{1}{8},{\sfrac{1}{8}}}$ .
\end{lemma}

%
%



\begin{proof}
Let $C_1,\dots,C_7$ be an instance of $\succETRc{1}$. 
Viewed as a $\succETR$ instance, $C_1,\dots,C_7$ encode a big conjunction $\varphi$
of quadratic equations. The bit length of $\varphi$ is bounded by $2^S$
where $S$ is the sum of the sizes of $C_1,\dots,C_7$. This is also an upper bound
for the number of variables. 

Let $K := \lceil 8 \cdot 2^S \cdot S + 3\rceil$. We first create the constant 
$\epsilon := 2^{-2^{K}}$. This can be done by repeated squaring:
\begin{align*}
  v_{1} & = 1/8 \\
  v_{2} + v_2 & = v_1 \\
  v_{i} \cdot v_{i} & = v_{i+1}, \quad 2 \le i \le K.
\end{align*}
It is a straightforward induction to prove that $v_{i} = 2^{-2^{i}}$ for $i \ge 2$.
The number of variables needed is $K$, which is exponential in the input
size $S$. However, the variables can be labeled by bit strings of length $O(S)$.
So we can construct circuits $D_1,\dots,D_7$ of size $\poly(S)$,
which output this set of equations 
(in the sense of $\succETRcR{\sfrac{1}{8}}{-\sfrac{1}{8},{\sfrac{1}{8}}}$).

Next, we transform the equations of the $\succETRc 1$ instance. For each variable
$x$, we have two variables $x'$ and $x''$ in the 
$\succETRcR{\sfrac{1}{8}}{-\sfrac{1}{8},{\sfrac{1}{8}}}$ instance.
An equation of the form $x = 1$ is replaced by $x' = v_K$.
An equation of the form $y + z = x$ is replaced by
$y' + z' = x'$. An equation of the form $yz = x$ is replaced by two equations 
$y'z' = x''$ and $v_K x' = x''$.  
It is easy to modify the circuits $C_1,\dots,C_7$ to obtain circuits
$C_1',\dots,C_7'$ encoding these new equations.
Then these circuits have to be combined with the circuits 
$D_1,\dots,D_7$ to encode the union of all equations.

When we have an assignment that 
satisfies the $\succETRc 1$ instance and assigns the value $\xi$ to $x$,
then we get a satisfying assignment of the 
$\succETRcR{\sfrac{1}{8}}{-\sfrac{1}{8},{\sfrac{1}{8}}}$ instance
by setting $x' = \epsilon \xi$ and $x'' = \epsilon^2 \xi$, where $\epsilon = 2^{-2^K}$. 
The converse is also true. Thus our reduction simply scales a solution
by $\epsilon$.

By Lemma \ref{lem:ETR:sol:size}, if the original $\succETRc 1$
instance has a solution, then we can assume that the size of each entry is bounded by 
$2^{2^{8 \cdot 2^S \cdot S}}$. By the choice of $K$, the new 
$\succETRcR{\sfrac{1}{8}}{-\sfrac{1}{8},{\sfrac{1}{8}}}$ instance has
a solution with entry sizes bounded by $1/8$.
On the other hand, if the original $\succETRc 1$ has no solution, then the new
instance does not have either.
\end{proof}



\end{document}